\documentclass{article}

\usepackage{etoolbox}
\newtoggle{preprint}
\newcommand{\preprint}[1]{\iftoggle{preprint}{#1}{}}
\newcommand{\neurips}[1]{\iftoggle{preprint}{}{#1}}

\toggletrue{preprint} %
\neurips{\usepackage[final,nonatbib]{neurips_2021}}
\preprint{\usepackage[preprint,nonatbib]{neurips_2021}}

\usepackage[utf8]{inputenc} %
\usepackage[T1]{fontenc}    %
\usepackage{url}            %
\usepackage{booktabs}       %
\usepackage{amsfonts}       %
\usepackage{nicefrac}       %
\usepackage{microtype}      %
\usepackage{crossreftools}      %

\usepackage{
	amsmath,
	amsthm,
	amssymb,
	wrapfig,
	cases,
	mathtools,
	thmtools,
	array,
	bbm,
	bm,
	subfigure,
	makecell,
	esvect,
	mathrsfs,
	breqn,
	booktabs,
	upgreek,
	changepage,
	listings,
	multirow,
	xargs,
	xstring,
	multicol,
	graphicx,
	float,
	xcolor,
	algpseudocode,
	enumitem,
	indentfirst,
	ifthen,
	wasysym,
	tikz,
	pgf,
	}

\usepackage[utf8]{inputenc}
\usepackage[ruled,vlined]{algorithm2e}

\usepackage[%
  minnames=99,maxnames=99,maxcitenames=99,
  style=alphabetic,
  sortcites=false, %
  uniquename=false,
  giveninits=true,
  hyperref,natbib,backend=biber]{biblatex}

\renewbibmacro{in:}{%
	\ifentrytype{article}{}{\printtext{\bibstring{in}\intitlepunct}}}

\DeclareLabelalphaTemplate{
  \labelelement{
    \field{label}
    \field[strwidth=1,strside=left]{labelname}
  }
  \labelelement{
    \field[strwidth=2,strside=right]{year}    
  }
}

\usetikzlibrary{shapes,arrows,trees,automata,positioning}

\usepackage{hyperref}
\hypersetup{%
  colorlinks,allcolors=[RGB]{0 0 128},bookmarksnumbered, bookmarksopen=true, bookmarksopenlevel=1,%
}

\usepackage[capitalize,nameinlink]{cleveref}
\usepackage{crossreftools}
\crefname{subsection}{Subsection}{Subsections}
\crefname{lemma}{Lemma}{Lemmas}
\crefname{corollary}{Corollary}{Corollaries}
\crefname{theorem}{Theorem}{Theorems}

\declaretheorem[name=Theorem]{theorem}

\declaretheorem[name=Lemma]{lemma}

\declaretheorem[name=Corollary]{corollary}

\declaretheorem[name=Assumption, numbered=no]{assumption*}
\declaretheorem[qed=$\triangleleft$,name=Example]{example}

\makeatletter
\renewcommand{\maketag@@@}[1]{\hbox{\m@th\normalsize\normalfont#1}}%
\makeatother

\makeatletter
\let\reftagform@=\tagform@
\def\tagform@#1{\maketag@@@{\ignorespaces\textcolor{gray}{(#1)}\unskip\@@italiccorr}}
\renewcommand{\eqref}[1]{\textup{\reftagform@{\ref{#1}}}}
\makeatother

\newcommand{\EE}{\mathbb{E}}

\newcommand{\II}{\mathbb{I}}

\newcommand{\MM}{\mathbb{M}}
\newcommand{\NN}{\mathbb{N}}

\newcommand{\RR}{\mathbb{R}}

\newcommand{\Dd}{\mathcal{D}}

\newcommand{\Ll}{\mathcal{L}}
\newcommand{\Mm}{\mathcal{M}}

\newcommand{\Oo}{\mathcal{O}}

\newcommand{\Rr}{\mathcal{R}}
\newcommand{\Ss}{\mathcal{S}}

\newcommand{\Xx}{\mathcal{X}}
\newcommand{\Yy}{\mathcal{Y}}

\newcommand{\one}{\mathbf{1}}

\def\[#1\]{\begin{equation}\begin{aligned}#1\end{aligned}\end{equation}}
\def\*[#1\]{\begin{equation*}\begin{aligned}#1\end{aligned}\end{equation*}}

\def\s*[#1\s]{\small\begin{align*}#1\end{align*}\normalsize}

\newcommand{\lcr}[3]{\left #1 #2 \right #3} %
\newcommand{\lcrx}[4][{-1}]{ 
	\IfEq{#1}{-1}{\left #2 {{{{#3}}}} \right #4}{
   	\IfEq{#1}{0}{#2 {{{{#3}}}} #4}{
	\IfEq{#1}{1}{\bigl #2 {{{{#3}}}} \bigr #4}{
	\IfEq{#1}{2}{\Bigl #2 {{{{#3}}}} \Bigr #4}{
	\IfEq{#1}{3}{\biggl #2 {{{{#3}}}} \biggr #4}{
	\IfEq{#1}{4}{\Biggl #2 {{{{#3}}}} \Biggr #4}{
    \GenericWarning{"4th argument to lcrx must be -1, 0, 1, 2, 3, or 4"}
    }}}}}}} %

\newcommand{\stk}[2]{\ensuremath{\stackrel{\text{#2}}{#1}}}

\newcommand{\upper}[1]{^{(#1)}}

\newcommand{\KL}[2]{\mathrm{KL}\rbra{#1 \ \Vert \ #2}}

\renewcommand{\Pr}{\mathbb{P}} %
\DeclareMathOperator*{\Var}{Var} %
\newcommand{\VVar}[1]{\underset{#1}{\text{\upshape Var}}} %

\newcommand{\setdelim}{: \ }

\newcommand{\Biggsetdelim}{: \ }
\newcommand{\iid}{i.i.d.}

\newcommand{\stT}{\ \text{s.t.}\ }

\newcommand{\ind}[2][{-1}]{\II{\sbra[#1]{{#2}}}} %

\def\multiset#1#2{\ensuremath{\left(\kern-.3em\left(\genfrac{}{}{0pt}{}{#1}{#2}\right)\kern-.3em\right)}}
\DeclareMathOperator{\convhull}{Conv}

\DeclareMathOperator*{\argmin}{\arg\min} %
\DeclareMathOperator*{\argmax}{\arg\max} %
\DeclareMathOperator*{\esssup}{\text{ess}\sup} %
\DeclareMathOperator*{\essinf}{\text{ess}\inf} %

\DeclareMathOperator*{\newlim}{\mathrm{lim}\vphantom{\mathrm{infsup}}}
\DeclareMathOperator*{\newmin}{\mathrm{min}\vphantom{\mathrm{infsup}}}
\DeclareMathOperator*{\newmax}{\mathrm{max}\vphantom{\mathrm{infsup}}}
\DeclareMathOperator*{\newinf}{\mathrm{inf}\vphantom{\mathrm{infsup}}}
\DeclareMathOperator*{\newsup}{\mathrm{sup}\vphantom{\mathrm{infsup}}}
\renewcommand{\lim}{\newlim}
\renewcommand{\min}{\newmin}
\renewcommand{\max}{\newmax}
\renewcommand{\inf}{\newinf}
\renewcommand{\sup}{\newsup}

\newcommand{\dee}{\mathrm{d}} %
\newcommand{\rnderiv}[2]{\frac{\dee #1}{\dee #2}}
\newcommand{\rnderivflat}[2]{\dee #1 / \dee #2}

\newcommand{\distas}{\sim}
\newcommand{\distiidas}{\stk{\distas}{iid}}

\newcommand{\binomialdist}{\mathrm{Bin}}
\newcommand{\bernoullidist}{\mathrm{Ber}}

\newcommand{\normaldist}{\mathrm{Nor}}

\newcommand{\uniformdist}{\mathrm{Unif}}

\newcommand{\rbra}[2][{-1}]{\lcrx[#1] ( {#2} ) }
\newcommand{\cbra}[2][{-1}]{\lcrx[#1] \{ {#2} \} }
\newcommand{\sbra}[2][{-1}]{\lcrx[#1] [ {#2} ] }

\newcommand{\abs}[2][{-1}]{\lcrx[#1] \vert {#2} \vert }
\newcommand{\set}[2][{-1}]{\lcrx[#1] \{ {#2} \}}
\newcommand{\floor}[2][{-1}]{\lcrx \lfloor {#2} \rfloor}
\newcommand{\ceil}[2][{-1}]{\lcrx[#1] \lceil {#2} \rceil}

\newcommand{\inner}[3][{-1}]{\lcrx[#1] \langle {{#2},\ {#3}} \rangle}
\newcommand{\innermeas}[4][{-1}]{\lcrx[#1] \langle {{#2},\ {#3}} \rangle_{#4}}
\newcommand{\card}[2][{-1}]{\lcrx[#1] \vert {#2} \vert }

\newcommand{\defas}{\overset{\mathrm{def}}{=}}  %

\newcommand{\Nats}{\NN}

\newcommand{\Reals}{\RR}
\newcommand{\ExtReals}{\overline{\Reals}}

\newcommand{\PosReals}{\Reals_+}
\newcommand{\PosPosReals}{\Reals_{++}}

\newcommand{\range}[2][{1}]{
	\IfEq{#1}{1}{\sbra{#2}}{\sbra{#2}_{#1}}}
\newcommand{\rangeO}[2][{0}]{
	\IfEq{#1}{0}{\sbra{#2}_0}{\sbra{#2}_{#1}}}

\newcommand{\ointer}[2][{-1}]{\lcrx[#1] ( {#2} ) }
\newcommand{\cinter}[2][{-1}]{\lcrx[#1] [ {#2} ] }

\newcommand{\union}{\cup}
\newcommand{\Union}{\bigcup}

\DeclareMathOperator{\interior}{interior}

\newcommand{\lrname}{regularizer scaling}

\newcommand{\simp}{\textup{\texttt{simp}}}

\newcommand{\entropy}{H}

\newcommand{\losssym}{\ell}

\newcommand{\Losssym}{L}
\newcommand{\lossvec}[1]{\losssym_{#1}} %
\newcommand{\loss}[2]{\losssym_{#1}(#2)} %
\newcommand{\Lossvec}[1]{\Losssym_{#1}} %
\newcommand{\tildeLossvec}[1]{{\tilde \Losssym}_{#1}}
\newcommand{\Loss}[2]{\Losssym_{#1}(#2)} %

\newcommand{\lossdistn}[1]{p_{#1}}
\newcommand{\epsidx}{i_{\epsilon}}

\newcommand{\numexperts}{N}
\newcommand{\experts}{\range{\numexperts}}
\newcommand{\distnball}[1]{\Dd_{#1}}

\newcommand{\distnballbrackets}{}

\newcommand{\effexperts}{\mathcal{I}_{0}\distnballbrackets} %
\newcommand{\effexpertstime}[1]{\mathcal{I}^{(#1)}_{0}\distnballbrackets} 
\newcommand{\neffexperts}{\experts\setminus\effexperts} %

\newcommand{\numeffexperts}{N_{0}\distnballbrackets}
\newcommand{\numeffexpertstime}[1]{N^{(#1)}_{0}\distnballbrackets}

\newcommand{\timeexp}[1]{T_{\expidx}\distnballbrackets}
\newcommand{\timeeffexp}{T_{0}\distnballbrackets}
\newcommand{\timeordexp}[1]{T_{(#1)}\distnballbrackets}

\newcommand{\Deltaexp}[1]{\Delta_{#1}\distnballbrackets}
\newcommand{\Deltaeff}{\Delta_{0}\distnballbrackets}
\newcommand{\Deltaordexp}[1]{\Delta_{(#1)}\distnballbrackets}

\newcommand{\expsumidx}{j}
\newcommand{\expidx}{i}
\newcommand{\expidxdum}{i'}
\newcommand{\Expidx}{I}
\newcommand{\Expidxoptpath}[1]{I^*_{#1}}

\newcommand{\effexpidx}{i_0}

\newcommand{\weightsym}{w}
\newcommand{\weightvec}[1]{\weightsym_{#1}} %
\newcommand{\weight}[2]{\weightsym_{#1}(#2)} %
\newcommand{\weightdumsym}{v}
\newcommand{\weightdumvec}[1]{\weightdumsym_{#1}} %
\newcommand{\weightdum}[2]{\weightdumsym_{#1}(#2)} %
\newcommand{\intweightsym}{v}
\newcommand{\intweightvec}[1]{\intweightsym_{#1}} %
\newcommand{\intweight}[2]{\intweightsym_{#1}(#2)} %

\newcommand{\lrsym}{\eta}
\newcommand{\lr}[1]{\lrsym_{#1}} %

\newcommand{\regretsym}{R}
\newcommand{\regret}[2]{\regretsym_{#1}(#2)} %
\newcommand{\bestregret}[1]{\regretsym_{#1}} %

\newcommand{\orderidx}[1]{\delta_{(#1)}}
\newcommand{\unorderidx}[1]{\delta_{#1}}

\newcommand{\quanteps}{\epsilon}

\newcommand{\expspace}{{\smash{\predspace}\vphantom{\dataspace}}^\numexperts}

\newcommand{\dataspace}{\Yy}
\newcommand{\predspace}{\hat\Yy}

\newcommand{\meas}{\Mm}
\newcommand{\finitemeas}{\Mm^\infty}
\newcommand{\probmeas}{\Mm}
\newcommand{\probmeasll}[1]{\Mm_{#1}}

\newcommand{\acronymstyle}[1]{{\upshape #1}}
\newcommand{\mathacronymstyle}[1]{{\scriptscriptstyle \text{\upshape #1}}}

\newcommand{\OGHedge}{\acronymstyle{Hedge}}
\newcommand{\SQUINT}{\acronymstyle{Squint}}
\newcommand{\COINBET}{\acronymstyle{CoinBetting}}
\newcommand{\ADAHedge}{\acronymstyle{AdaHedge}}

\newcommand{\FTRL}{\acronymstyle{FTRL}}
\newcommand{\CARE}{\acronymstyle{CARE}}
\newcommand{\FTRLCARE}{\FTRL-\CARE}

\newcommand{\MetaCARE}{\acronymstyle{Meta-CARE}}

\newcommand{\ABN}{\acronymstyle{FTRL-CARL}}
\newcommand{\ABNreg}{\acronymstyle{CARL}}
\newcommand{\shortABN}{\mathacronymstyle{CARL}}
\newcommand{\ABNweight}[2]{w^{\shortABN}_{#1}(#2)}
\newcommand{\ABNweightvec}[1]{w^{\shortABN}_{#1}}
\newcommand{\underweightmeasfun}{\gamma}

\newcommand{\ABNC}{\acronymstyle{abNormal}}

\newcommand{\NH}{\acronymstyle{NormalHedge}}

\newcommand{\bregmansym}{B}
\newcommand{\bregman}[3]{\bregmansym_{#1}(#2; #3)}

\newcommand{\partialentropy}{h}
\newcommand{\Bchar}{C}
\newcommand{\Achar}{A}

\newcommand{\partialentropyB}{\partialentropy_\Bchar}
\newcommand{\partialentropyA}{\partialentropy_\Achar}
\newcommand{\entropyB}{\entropy_\Bchar}
\newcommand{\entropyA}{\entropy_\Achar}

\DeclareMathOperator{\erf}{erf}

\newcommand{\const}{c}
\newcommand{\Const}{C}

\newcommand{\konst}{k}

\newcommand{\abdivfun}{h}

\newcommand{\fdivergence}[3]{D_{#1}\rbra{\lcr.{#2\vphantom{#3}}\Vert #3}}
\newcommand{\fdivfun}{f}

\newcommand{\ftrlobj}[1]{F_{#1}}

\newcommand{\fconstA}{c_1}
\newcommand{\fconstB}{c_2}

\newcommand{\convexset}{V}

\newcommand{\convextargetelem}{u}

\newcommand{\genregularizer}[1]{\psi_{#1}}
\newcommand{\gendivfun}{f}
\newcommand{\gendiv}[3]{\Psi_{\!#1}^{#2}(#3)} %
\newcommand{\gendivrn}[3]{\Psi_{\!#1}^{#2}\rbra{#3}} %
\newcommand{\gendivnoarg}[2]{\Psi_{\!#1}^{#2}}%
\newcommand{\gennormfun}{g}

\newcommand{\gentruncfun}{\tau_{\fdivfun'}}
\newcommand{\genupperdiv}{M_{\fdivfun'}}
\newcommand{\genlowerdiv}{m_{\fdivfun'}}
\newcommand{\basemeasmin}{\underline{\basemeasvec}}

\newcommand{\basemeasvec}{\nu}
\newcommand{\basemeas}[1]{\nu(#1)}

\newcommand{\baseprobvec}{\overline{\nu}}
\newcommand{\baseprob}[1]{\overline{\nu}(#1)}

\newcommand{\genexpset}{\Theta}
\newcommand{\genexpSF}{\Sigma}
\newcommand{\genexpSFelem}{A}
\newcommand{\genexpidx}{\theta}
\newcommand{\genexpidxdum}{\theta'}

\newcommand{\genexpspace}{(\genexpset,\genexpSF)}

\newcommand{\genweightsym}{\mu}

\newcommand{\genweightvec}[1]{\genweightsym_{#1}} %
\newcommand{\genpriorvec}{\pi}
\newcommand{\genprior}[1]{\pi(#1)}

\newcommand{\genrnsym}{x}
\newcommand{\genrn}[2]{\genrnsym_{#1}(#2)} %
\newcommand{\genrnvec}[1]{\genrnsym_{#1}} %

\newcommand{\genmeasvec}[2]{{#1}\upper{\genrnvec{#2}}}

\newcommand{\genrndumsym}{z}
\newcommand{\genrndum}[2]{\genrndumsym_{#1}(#2)} %
\newcommand{\genrndumvec}[1]{\genrndumsym_{#1}} %

\newcommand{\intgenrnsym}{{\hat \genrndumsym}}
\newcommand{\intgenrn}[2]{\intgenrnsym_{#1}(#2)} %
\newcommand{\intgenrnvec}[1]{\intgenrnsym_{#1}} %

\newcommand{\intgenmeasvec}[2]{{#1}\upper{\intgenrnvec{#2}}}
\newcommand{\inttgenrnsym}{{\tilde \genrndumsym}}
\newcommand{\inttgenrn}[2]{\inttgenrnsym_{#1}(#2)} %
\newcommand{\inttgenrnvec}[1]{\inttgenrnsym_{#1}} %

\newcommand{\orthgenrnsym}{{\tilde \genrnsym}}
\newcommand{\orthgenrn}[2]{\orthgenrnsym_{#1}(#2)} %
\newcommand{\orthgenrnvec}[1]{\orthgenrnsym_{#1}} %
\newcommand{\intgenrnbothsym}{{\overline \genrndumsym}}
\newcommand{\intgenrnboth}[2]{\intgenrnbothsym_{#1}(#2)} %
\newcommand{\intgenrnbothvec}[1]{\intgenrnbothsym_{#1}} %

\newcommand{\gentargetvec}{q}
\newcommand{\gentarget}[1]{q(#1)}

\newcommand{\contexpidx}{\theta}
\newcommand{\contexpidxdum}{\vartheta}

\newcommand{\contpriorvec}{\pi}

\newcommand{\contrnsym}{x}
\newcommand{\contrn}[2]{\contrnsym_{#1}(#2)} %
\newcommand{\contrnvec}[1]{\contrnsym_{#1}} %

\newcommand{\contrndumsym}{z}

\newcommand{\intcontrnsym}{{\hat \contrndumsym}}
\newcommand{\intcontrnvec}[1]{\intcontrnsym_{#1}} %

\newcommand{\intcontmeasvec}[2]{{#1}\upper{\intcontrnvec{#2}}}

\newcommand{\rndumidx}{x}

\newcommand{\conttargetvec}{q}

\newcommand{\measfuns}{\MM}
\newcommand{\borelSF}{\Rr}

\newcommand{\mean}[1]{m_{#1}}

\newcommand{\boundrvspace}{\Ll^{\infty}}
\newcommand{\bounddrvspace}{\Ll_{[0,1]}^{\infty}}
\newcommand{\rvspace}{\Ll^{1}}

\newcommand{\posrvspace}{\Ll_{+}^{1}}
\newcommand{\probrvspace}{\Xx}

\newcommand{\funcdiffsym}{\delta}
\newcommand{\funcdiff}[3]{\funcdiffsym {#1}[{#2}; {#3}]}

\newcommand{\Losslower}{\underline{\Losssym}}
\newcommand{\Lossupper}{\overline{\Losssym}}

\newcommand{\normalccdf}{\overline{\Phi}}
\newcommand{\normalcdf}{\Phi}
\newcommand{\normalpdf}{\phi}

\newcommand{\empiricalcdf}[1]{{\widehat F}_{#1}}

\DeclareMathOperator{\domfdivfun}{\Ss}

\usepackage[hide=false,setmargin=true,marginparwidth=1in]{marginalia}

\title{Minimax Optimal Quantile and Semi-Adversarial Regret via Root-Logarithmic Regularizers}

\author{%
Jeffrey Negrea\thanks{Equal contribution authors.}\\
University of Toronto\\
\texttt{jeffrey.negrea@mail.utoronto.ca}
\And
Blair Bilodeau${}^{*}$\\
University of Toronto\\
\texttt{blair.bilodeau@mail.utoronto.ca}
\And
Nicol\`o Campolongo\\
Spanflug Technologies GmbH\\
\texttt{nico.campolongo@spanflug.de}
\And
Francesco Orabona\\
Boston University\\
\texttt{francesco@orabona.com}
\And
Daniel M. Roy\\
University of Toronto\\
\texttt{daniel.roy@utoronto.ca}
}

\begin{document}

\maketitle

\begin{abstract}
Quantile (and, more generally, KL) regret bounds,
such as those achieved by \NH{} (\citeauthor{chaudhuri2009} \citeyear{chaudhuri2009}) and its variants,
relax the goal of competing against the best individual expert to only competing against a majority of experts on adversarial data.
More recently,
the semi-adversarial paradigm (\citeauthor{semiadv} \citeyear{semiadv}) provides an alternative relaxation of adversarial online learning by considering data that may be neither fully adversarial nor stochastic (\iid{}). 
We achieve the minimax optimal regret in both paradigms using \FTRL{} with separate, novel, root-logarithmic regularizers,
both of which can be interpreted as yielding variants of \NH{}.
We extend existing KL regret upper bounds, which hold uniformly over target distributions, to possibly uncountable expert classes with arbitrary priors; provide the first full-information lower bounds for quantile regret on finite expert classes (which are tight); and provide an adaptively minimax optimal algorithm for the semi-adversarial paradigm that adapts to the true,
unknown constraint faster,
leading to uniformly improved regret bounds over existing methods.

\end{abstract}

\section{Introduction}

We focus on the setting of learning with expert advice~\citep{Vovk90,littlestone94},
where in each round the learner selects a probability distribution over experts, observes the loss of each expert, and incurs the average loss of the experts under the learner's selected distribution.
The learner's objective is to minimize regret against some mixture of the experts, which is the difference between their cumulative loss and the cumulative loss of the expert mixture over $T$ rounds.

The classical ``worst-case'' online learning paradigm assumes that the losses are adversarial---that is, they are chosen to make the learner perform as poorly as possible---and demands that the learner competes against the best-performing expert.
However,
there are many real-world settings where this assumption is too pessimistic, 
and consequently we focus on designing algorithms with provable guarantees that adapt to easier notions of both data and performance measures.
A non-exhaustive list of work on ``easy data'' includes \citet{freund97}, \citet{cesabianchi2007secondorder}, \citet{vanerven11adahedge}, and \citet{gaillard14}, all of which use variants of the \OGHedge{} algorithm \citep{freund97} to obtain regret bounds in terms of data-dependent quantities. Ideally, such quantities are small when the data are ``easy'' to predict.

In this work, we focus on two paradigms beyond the classical worst-case: first, we consider relaxing the performance measure to \emph{quantile (KL) regret}, which measures the ability of an algorithm to compete against an unknown mixture of the experts that potentially performs worse than a point-mass on the single best expert, and second, we consider regret within the \emph{semi-adversarial paradigm}, which defines a spectrum of constraints on the permissible data distributions between stochastic and adversarial.
The concept of $\quanteps$-quantile regret was introduced by \citet{chaudhuri2009}, in which the player competes against the $\lfloor \quanteps \numexperts \rfloor$ best experts (out of $\numexperts$ total) rather than the single best.
The authors demonstrated empirically that \OGHedge{} does poorly in this paradigm, and introduced a new algorithm \NH{} with an upper bound on quantile regret of $\sqrt{(T+\smash{(\log\numexperts)^2})(1+\log (1/\quanteps))}$. Later algorithms improved it to $\sqrt{ T (1+\log(1/\quanteps))}$~\citep{ChernovV10,orabona2016}, removing the dependency on $\numexperts$.
The semi-adversarial paradigm considers constraining the adversary's choice of data distributions, which was first motivated by \citet{rakhlin2011online}. \citet{semiadv} extended this idea, defining \emph{adaptive} minimax regret with respect to such constraints and providing an efficient algorithm with corresponding regret bounds.

\noindent\textbf{Contributions} \
While the best known algorithms for the above two paradigms are intrinsically different, we show that the follow-the-regularized-leader (\FTRL{}) algorithm with new root-logarithmic regularizers 
achieves minimax optimal performance for
both quantile and semi-adversarial regret.
First, 
we provide the first \FTRL{} algorithm with minimax optimal quantile regret guarantees, and do so without using the additive normalization step of previous algorithms~\citep[Section 2.1]{plg07}.
We achieve root-KL bounds that hold uniformly over target distributions on (possibly uncountable) expert classes with arbitrary priors, and reduce to the optimal quantile regret for discrete uniform priors.
Moreover,
we prove matching lower bounds for quantile regret with finite expert classes, demonstrating the optimality of known upper bounds (including our own) for the first time.
Finally, 
in the semi-adversarial paradigm, we improve the dependence on the number of experts in the regret bound, obtaining uniformly improved upper bounds over previous work.

We achieve the above results through a novel local-norm analysis of \FTRL{} with \emph{linearly decomposable} regularizers on general (possibly uncountable) expert spaces. We use this analysis in conjunction with basic conditions on the first and second derivatives to design and analyze the root-logarithmic regularizers.
We believe that this approach is fundamentally different from existing ones and could lead to further advances in obtaining optimal algorithms. In fact, there exist results stating if a regret bound is achievable by some algorithm, then that same bound is nearly achievable by mirror descent with some potential function (see, e.g., \citet[Thm 9]{srebro2011universality}). However, it is not clear how to design such a function in practice. 
In contrast, our general \FTRL{} bound reduces the choice of regularizer to a single univariate function, and clarifies how fundamental 
properties of this function lead to trade-offs in the regret bound.

\noindent\textbf{Related Work} \ 
\citet{ChernovV10} first discussed the fact that the $\quanteps$-quantile regret corresponds to the KL divergence between a uniform prior and the uniform mixture of the top $\lfloor \quanteps \numexperts \rfloor$ competitors, and such KL bounds have consequently quickly followed quantile regret bounds. 
\citet{luo2014,luo2015} provided a variant of \NH{} along with a data-dependent KL regret bound at the cost of an additional $\log T$ factor; a similar but even tighter result was obtained by \citet{koolen2015}. \citet{orabona2016} showed that these algorithms can be obtained with a reduction to optimal coin-betting online algorithms, yet none of them can be reduced to an \FTRL{} algorithm.
Independently, \citet{HarveyLPR20} proved that a similar strategy allows one to achieve the optimal anytime regret for the setting with $\numexperts=2$ experts. 

The first alternative to additively normalized algorithms 
in the (bandit) learning with experts setting was the INF algorithm \citep{audibert09}, which was later recast as online mirror descent \citep{AudibertBL11}.
\citet{alquier20unbounded} also did not use additive normalization, and first introduced \FTRL{} with $\fdivfun$-divergences (focusing on the $\chi^2$-divergence) along with regret bounds for continuous distributions with certain unbounded loss functions. For the specific case of \OGHedge{} (and its online mirror descent analogue), an analysis on continuous spaces was given by \citet{krichene15continuum} and then more generally by \citet{hoeven2018many}, while a coarser analysis for $\gendivfun$-divergences was given by \citet{alquier20unbounded}.
Finally, choosing regularizers that are tuned to minimize regret for specific tasks has recently led to advances in the online learning literature, both with bandit feedback \citep{agarwal2017corralling,wei2018more,foster2016learning} and in the full-information setting \citep{luo2018efficient}. 

The best known asymptotic lower bound for learning with expert advice is by \citet{cesa-bianchi1997}, while a finite-time lower bound that asymptotically matches the leading constant was proved by \citet{OrabonaP15}.
We are not aware of lower bounds for the quantile regret, with the notable
exception of \citet{Koolen13}, who proved a lower bound on the regret of
learning with two experts based on the KL divergence against some prior.

\section{Notation}\label{sec:notation}

To analyze \FTRL{} beyond the finite setting requires some more care, and we rely on the language of measure theory to handle finite and uncountable expert classes simultaneously.
Let $\genexpspace$ be a measurable space, $\finitemeas\genexpspace$ and $\probmeas\genexpspace$ denote the collection of finite and probability measures respectively, and $\basemeasvec\in\finitemeas\genexpspace$ be arbitrary.
A measure $\gentargetvec\in\finitemeas\genexpspace$ is \emph{absolutely continuous} with respect to $\basemeasvec$ (denoted $\gentargetvec \ll \basemeasvec$) if $\gentarget{\genexpSFelem} = 0$ for all $\genexpSFelem\in\genexpSF$ such that $\basemeas{\genexpSFelem} = 0$.
Let $\probmeasll{\basemeasvec}\genexpspace = \{\gentargetvec\in\probmeas\genexpspace: \gentargetvec\ll\basemeasvec\}$ be the set of probability measures that are absolutely continuous with respect to $\basemeasvec$.
For any integer $\numexperts$, let $[\numexperts] = \{1,\dots,\numexperts\}$, and $\simp([\numexperts]) = \{p \in [0,1]^\numexperts \setdelim \textstyle\sum_{\expidx=1}^\numexperts p_\expidx = 1\}$.
Let $\PosReals = [0,\infty)$ and $\PosPosReals = (0,\infty)$, and define $\measfuns \equiv \measfuns(\genexpspace,(\Reals,\borelSF))$ to be the space of measurable functions from $\genexpspace$ to $(\Reals,\borelSF)$, where $\borelSF$ is the Borel $\sigma$-field on $\Reals$. 
Define sets of bounded measurable functions
\*[
	\boundrvspace 
		= \set[1]{\lossvec{}\in\measfuns \setdelim \sup_{\genexpidx\in\genexpset}\abs[0]{\loss{}{\genexpidx}} < \infty} \ &\text{ and } \
	\bounddrvspace 
		= \set[1]{\lossvec{}\in\measfuns \setdelim 0 \leq \lossvec{} \leq 1},
\]
sets of integrable functions
\*[
	\rvspace(\basemeasvec) 
		= \set[1]{\genrnvec{} \in \measfuns \setdelim \textstyle \int \abs{\genrn{}{\genexpidx}} \basemeas{\dee \genexpidx} < \infty} \ &\text{ and } \
	\posrvspace(\basemeasvec) 
		= \set[1]{\genrnvec{} \in \rvspace(\basemeasvec) \setdelim \genrnvec{} \geq 0},
\]
and the set of Radon--Nikodym derivatives (w.r.t.\ $\basemeasvec$) of probability measures
\*[
	\probrvspace(\basemeasvec) 
		= \set[1]{\genrnvec{} \in \posrvspace(\basemeasvec) \setdelim \textstyle \int \genrn{}{\genexpidx} \basemeas{\dee \genexpidx} =1}.
\]
For every $\genrnvec{}\in\probrvspace(\basemeasvec)$, let $\genmeasvec{\basemeasvec}{} \in \probmeasll{\basemeasvec}\genexpspace$ denote the unique probability measure satisfying $\rnderivflat{\genmeasvec{\basemeasvec}{}}{\basemeasvec} = \genrnvec{}$. 
For every $\genrnvec{},\genrndumvec{}\in\boundrvspace$, let $\convhull(\genrnvec{},\genrndumvec{})
	= \{\alpha \genrnvec{} + (1-\alpha)\genrndumvec{} \setdelim \alpha \in [0,1]\}$ denote their convex hull.
For every $\lossvec{} \in \boundrvspace$ and $\genrnvec{}\in\rvspace(\basemeasvec)$, define
\*[
  \innermeas{\lossvec{}}{\genrnvec{}}{\basemeasvec}
  = \int \loss{}{\genexpidx} \genrn{}{\genexpidx} \basemeas{\dee \genexpidx}.
\]

\noindent\textbf{Prediction with expert advice}
We consider the following setting of online linear optimization. 
For each round $t \in \range{T}$, the player selects $\genweightvec{t}\in\probmeas\genexpspace$ based only upon information available prior to round $t$, and then observes $\lossvec{t} \in \boundrvspace$. Performance is measured using the \emph{regret} against some probability measure $\gentargetvec\in\probmeas\genexpspace$, which is defined by
\*[
\regret{T}{\gentargetvec}
	& = \sum_{t=1}^T \EE_{\genexpidx\distas \genweightvec{t}}\sbra{\loss{t}{\genexpidx}} - \sum_{t=1}^T \EE_{\genexpidx\distas\gentargetvec}\sbra{\loss{t}{\genexpidx}}.
\]
The player's goal is to select $\genweightvec{t}$ so that the cumulative loss is not much larger than that of the average loss of an expert under $\gentargetvec$, and consequently to have small regret.

The elements of $\genexpset$ can be regarded as experts 
for some prediction problem, and the learner aggregates the predictions of the experts 
by selecting an expert at random according to $\genweightvec{t}$ at round $t$.
For concreteness, note that the usual prediction with expert advice setting for $\numexperts$ experts corresponds to $\probmeas\genexpspace= \simp(\experts)$.
Note that, for any convex loss, applying online linear optimization to the gradients of the losses provides an upper bound on the performance of online convex optimization.

\noindent\textbf{Follow-the-regularized-leader}
Let $\ExtReals = \Reals\union\set[0]{+\infty}$.
Follow-the-regularized-leader (\FTRL{}) \citep{Shalev-Shwartz07,AbernethyHR08,hazan2010variance} forms a broad class of algorithms for online convex optimization. 
For a finite measure $\basemeasvec\in\finitemeas\genexpspace$ and a sequence of regularizers $\rbra[0]{\genregularizer{t}:\probrvspace(\basemeasvec)\to\ExtReals\,}_{t\in\range{T}}$, $\rbra[0]{\genregularizer{t}}_{t\in\range{T}}$-regularized \FTRL{} is defined by selecting $\genweightvec{t+1} = \genmeasvec{\basemeasvec}{t+1}$ using
\[\label{eqn:ftrl-weights}
	\genrnvec{t+1}
 	\in \argmin_{\genrnvec{} \in \probrvspace(\basemeasvec)}
	\cbra[1]{
	\innermeas[0]{\Lossvec{t}}{\genrnvec{}}{\basemeasvec}
	+ \genregularizer{t+1}(\genrnvec{})},
\]
where $\Lossvec{0} = 0$ and $\Lossvec{t} = \sum_{s=1}^{t} \lossvec{s}$.

As mentioned above, a classical algorithm for prediction with expert advice on a finite class of $\numexperts$ experts is \OGHedge{}. While originally analyzed using potential functions, \OGHedge{} also corresponds to \FTRL{} where $\basemeasvec{}$ is the counting measure and $\genregularizer{t}(\cdot) = -\lr{t}^{-1} \entropy(\cdot)$ for the Shannon entropy $\entropy(\genrnvec{}) = \sum_{\expidx\in\range{\numexperts}} \genrn{}{\expidx}\log(1/\genrn{}{\expidx})$ with any sequence of \lrname{}s $(\lr{t})_{t\in\Nats} \subseteq \PosPosReals$.

Depending on whether the space is continuous or discrete, the range of elements of $\probrvspace(\basemeasvec)$ will change, and consequently also the minimal domain on which regularizers must be defined.
Let $ \basemeasmin = \inf\set{\basemeas{A} \setdelim A \in\genexpSF, \basemeas{A}>0 }$,
and let $\domfdivfun(\basemeasvec) = [0,1/\basemeasmin]$ if $\basemeasmin > 0$ and $\PosReals$ otherwise.
By definition,
$\basemeasvec(\{\genexpidx\in\genexpset:\genrn{}{\genexpidx}\in \domfdivfun(\basemeasvec)\})=1$
for all $\genrnvec{}\in\probrvspace(\basemeasvec)$. 

Concretely, when $\basemeasvec$ is counting measure then $\domfdivfun(\basemeasvec) = [0,1]$, when $\basemeasvec$ is uniform on $[\numexperts]$ then $\domfdivfun(\basemeasvec) = [0,\numexperts]$, and when $\basemeasvec$ is a continuous distribution then $\domfdivfun(\basemeasvec) = \PosReals$.
In this work, we consider \emph{linearly decomposable} regularizers of the form $\genregularizer{t}(\genrnvec{})=\lr{t}^{-1}\gendiv{\gendivfun}{\basemeasvec}{\genrnvec{}}$, where
\[ \label{eq:diagonal-regularizer}
	\gendiv{\gendivfun}{\basemeasvec}{\genrnvec{}} & = \int \gendivfun(\genrn{}{\genexpidx}) \basemeas{\dee\genexpidx}
\]
for some $\gendivfun: \domfdivfun(\basemeasvec) \to \Reals$.
We refer to the algorithm that selects $\genweightvec{t+1}$ using \cref{eqn:ftrl-weights} with a regularizer of the form in \cref{eq:diagonal-regularizer} as $\gendivnoarg{\gendivfun}{\basemeasvec}$-regularized \FTRL{} with \lrname{} $\lr{t}$.

\section{Applications of FTRL with linearly decomposable regularizers}\label{sec:quantile-main-results}

In \cref{sec:ftrl-gen}, we provide a general analysis of FTRL with linearly decomposable regularizers for arbitrary $\gendivfun$. 
First, we motivate such a general analysis by demonstrating the benefits of studying choices of $\gendivfun$ beyond traditional FTRL regularizers with multiple examples, including quantile regret. To do so, we state the following corollaries of our general FTRL regret bound (\cref{lem:gen-ftrl-decomp} in \cref{sec:ftrl-gen}) that achieve ``root-KL'' and variance bounds respectively. Proofs are deferred to \cref{sec:proof-continuous-KL}.

\begin{restatable}{corollary}{ContinuousKL}
\label{thm:continuous-KL}
Suppose $\basemeasvec\in\probmeas\genexpspace$ and $\fdivfun :\PosReals \to \Reals$ satisfies:
\begin{enumerate}[noitemsep]
\vspace{-5pt}
  \item $\fdivfun(1) =0$;
  \item $\fdivfun$ is twice continuously differentiable, $\fdivfun''>0$  on $\PosPosReals$, and $\fdivfun''(0+)>0$;
  \item either $\fdivfun$ or $1/\fdivfun''$ is increasing on $\PosPosReals$;
  \item there exist $\fconstA, \fconstB>0$ such that $\fdivfun'' \cdot (\fdivfun+\fconstA) \geq \fconstB$.
\end{enumerate}
For any sequence $(\lossvec{t})_{t\in\Nats} \subseteq \bounddrvspace$, $\gentargetvec\in\probmeasll{\basemeasvec}\genexpspace$, and $T$,
$\gendivnoarg{\gendivfun}{\basemeasvec}$-regularized \FTRL{} with \lrname{} $\lr{t}=\sqrt{\fconstB/t}$ achieves
\*[
  \regret{T}{\conttargetvec}
    &\leq \sqrt{\tfrac{T+1}{\fconstB}} \, 
    \gendiv{\gendivfun}{\basemeasvec}{\rnderivflat{\gentargetvec{}}{\basemeasvec}}
    + \tfrac{\fconstA}{\sqrt{\fconstB}}\sqrt{T}.
\]
\end{restatable}

The regret bound above can be easily turned into a uniform root-KL bound under the following additional assumption (the result follows from \cref{lem:KL-leq-sqrtDf} in \cref{sec:proof-continuous-KL}).

\begin{corollary}
\label{cor:kl-bound}
If, in addition to the assumptions of \cref{thm:continuous-KL}, there exist $\konst_1\in\Reals$ and $\konst_2\in\PosReals$ with $f(x)\leq \konst_1+\konst_2 x \sqrt{\log(1+x)}$,
then
\*[
  \regret{T}{\conttargetvec}
    &\leq \sqrt{\tfrac{T+1}{\fconstB}}\rbra{\konst_1 +\konst_2\sqrt{1+\KL{\conttargetvec}{\contpriorvec}}}
    + \tfrac{\fconstA}{\sqrt{\fconstB}}\sqrt{T}.
\]

\end{corollary}

Under different conditions on $\gendivfun$, 
we also derive variance bounds with respect to the (intermediate) predictive distributions, similar to AdaHedge \citep{vanerven11adahedge} and AdaFTRL \citep{orabona2015scale}, or with respect to a ``prior'', as in \citep{alquier20unbounded}. 
We state the conditions on the regularizer and the corresponding regret bounds in the following corollary.

\begin{restatable}{corollary}{VarianceBoundCoroll}
\label{fact:gen-variance-bound}
Suppose $\basemeasvec\in\finitemeas\genexpspace$ and $\fdivfun :\domfdivfun(\basemeasvec) \to \Reals$ satisfies:
\begin{enumerate}[noitemsep]
\vspace{-5pt}
  \item $\gendivfun(1/\basemeas{\genexpset})\geq 0$;
  \item $\fdivfun$ is twice continuously differentiable, $\fdivfun''>0$  on $\interior(\domfdivfun(\basemeasvec))$, and $\fdivfun''(0+)>0$.
\end{enumerate}
For any sequence $(\lossvec{t})_{t\in\Nats} \subseteq \bounddrvspace$, $\gentargetvec\in\probmeasll{\basemeasvec}\genexpspace$, and $T$,
$\gendivnoarg{\gendivfun}{\basemeasvec}$-regularized \FTRL{} achieves
\begin{enumerate}
	\vspace{-10pt}
\item If $1/\gendivfun''(x) \leq \Const x$, $\lr{t+1} = \rbra[1]{\Const\sbra[1]{1/2+\sum_{s=1}^{t-1} \Var_{\genexpidx \sim \intgenmeasvec{\basemeasvec}{s+1}} \loss{s}{\genexpidx}}}^{-1/2}$
gives
\*[
	\regret{T}{\gentargetvec{}}
	&\leq \sbra[1]{\gendiv{\gendivfun}{\basemeasvec}{\rnderivflat{\gentargetvec{}}{\basemeasvec}} + 1}  \sqrt{\Const\sbra[1]{1/2 + \textstyle\sum_{t=1}^{T} \Var_{\genexpidx \sim \intgenmeasvec{\basemeasvec}{t+1}} \loss{t}{\genexpidx}}}.
\]
\item If $1/\gendivfun'' \leq \Const$, $\lr{t+1} = \rbra[1]{\Const \, \basemeas{\genexpset}\sbra[1]{1/4+\sum_{s=1}^t \Var_{\genexpidx \sim \baseprobvec} \loss{s}{\genexpidx}}}^{-1/2}$
for $\baseprobvec = \basemeasvec / \basemeas{\genexpset}$ gives
\*[
	\regret{T}{\gentargetvec{}}
	&\leq \sbra[1]{\gendiv{\gendivfun}{\basemeasvec}{\rnderivflat{\gentargetvec{}}{\basemeasvec}}+1}  \sqrt{\Const \, \basemeas{\genexpset} \sbra[1]{1/4 + \textstyle\sum_{t=1}^{T} \Var_{\genexpidx \sim \baseprobvec} \loss{t}{\genexpidx}}}.
\]
\end{enumerate}
\end{restatable}

\subsection{Examples of $\fdivfun$-divergence \FTRL{}}

We now apply all three of these corollaries to more concrete choices of $\gendivfun$.
Note that in the setting of \cref{thm:continuous-KL}, since $\basemeasvec$ is a probability distribution and $\gendivfun$ is convex with $\gendivfun(1)=0$, the regularizer corresponds to an $\gendivfun$-divergence, which we denote by $\gendiv{\gendivfun}{\basemeasvec}{\rnderivflat{\gentargetvec{}}{\basemeasvec}} = \fdivergence{\fdivfun}{\gentargetvec}{\basemeasvec}$ for any $\gentargetvec\in\probmeasll{\basemeasvec}\genexpspace$. In this case, we call $\basemeasvec$ a \emph{prior} and refer to the algorithm as \emph{$\fdivfun$-divergence \FTRL{}}.
To demonstrate the utility of \cref{thm:continuous-KL}, we now show how our result recovers the classical analysis for \OGHedge{} and applies to our new root-logarithmic regularizer, both examples of $\fdivfun$-divergence \FTRL{}. We also use our general expression for the solution to \FTRL{} with linearly decomposable regularizers, given by \cref{lem:gen-ftrl-soln-formula} in \cref{sec:ftrl-gen}, to obtain explicit, novel expressions for some solutions of $\fdivfun$-divergence \FTRL{}.

\begin{example}[\OGHedge{}]
The regularizer corresponding to \OGHedge{} (and the Gibbs posterior from Bayesian inference) is given by $\fdivfun(x)= x\log x$, which satisfies $\fdivergence{\fdivfun}{\cdot}{\contpriorvec} = \KL{\cdot}{\contpriorvec}$ for any $\contpriorvec\in\probmeas\genexpspace$.
This choice satisfies the assumptions of \cref{thm:continuous-KL} with $c_1=1+1/e$ and $c_2=1$, since $1/\fdivfun''$ is increasing, but \emph{does not} satisfy the conditions of \cref{cor:kl-bound}.
Thus, we do not obtain uniform root-KL bounds for \OGHedge{}; we conjecture this is in fact a limitation of \OGHedge{} and not merely an artifact of our analysis.
We can, however, apply \cref{fact:gen-variance-bound} with $\Const = 1$ to obtain 
\*[
  \regret{T}{\gentargetvec}
  \leq \sbra[2]{\KL{\gentargetvec}{\basemeasvec}+1}\sqrt{1/2 + \textstyle\sum_{t=1}^{T} \Var_{\contexpidx\sim\intcontmeasvec{\basemeasvec}{t+1}} \loss{t}{\contexpidx}}.
\]
Further, by \cref{lem:gen-ftrl-soln-formula}, $\fdivfun$-divergence \FTRL{} with \lrname{} $\lr{t+1}$ recovers the familiar formula
\*[
  \hskip10em
	\contrn{t+1}{\contexpidx} 
    = \frac{\exp(-\lr{t+1}\Loss{t}{\contexpidx})}{\int \exp(-\lr{t+1}\Loss{t}{\contexpidxdum}) \basemeas{\dee \contexpidxdum}}. 
    \hskip12.5em
    \qedhere
\]
\end{example}

\begin{example}[$\chi^2$-divergence]
\citet{alquier20unbounded} analyzed $\fdivfun$-divergence \FTRL{} with constant \lrname{}s, obtaining variance bounds (Theorem~2.1) that require knowledge of the variance to tune the \lrname{}. He specifically focused on the KL-divergence, covered by the previous example, and the $\chi^2$-divergence, corresponding to $\fdivfun(x) = x^2-1$. This $\fdivfun$ clearly satisfies the conditions of \cref{thm:continuous-KL} with $\fconstA=\fconstB=2$,
so we match the optimized bound of \citet[Corollary~2.4]{alquier20unbounded} with
\*[
  \regret{T}{\conttargetvec}
    &\leq \sqrt{T} \, \chi^2(\conttargetvec \, \Vert \, \basemeasvec) + \sqrt{2T}.
\]
Further, the conditions of \cref{fact:gen-variance-bound} are satisfied with $\Const=1/2$, so we obtain the novel, potentially much smaller variance bound (without requiring advance knowledge of the variances)
\*[
  \regret{T}{\gentargetvec{}}
  &\leq \tfrac{1}{\sqrt{2}} \sbra{\chi^2(\conttargetvec \, \Vert \, \basemeasvec)+1}  \sqrt{1/4 + \textstyle\sum_{t=1}^{T} \Var_{\contexpidx \sim \basemeasvec} \loss{t}{\contexpidx}} \, .
\]
By \cref{lem:gen-ftrl-soln-formula}, the $\chi^2$-divergence \FTRL{} solution with \lrname{} $\lr{t+1}$ is
\*[
  \contrn{t+1}{\contexpidx} 
  & = \tfrac{1}{2}\rbra{\konst^*_{t+1}-\lr{t+1}\Loss{t}{\contexpidx}}_+
\]
where $\konst^*_{t+1}\in\Reals$ solves $\int  \sbra[1]{\tfrac{1}{2}\rbra[1]{\konst^*_{t+1}-\lr{t+1}\Loss{t}{\contexpidx}}_+} \basemeas{\dee \contexpidx} = 1$, which matches the formula obtained by \citet[Example~3.2]{alquier20unbounded}.
\end{example}

\begin{example}[\ABNC{}]
\label{example:root_log_regularizer}
We call $\fdivfun$-divergence \FTRL{} with any $\fdivfun$ satisfying the conditions of \cref{cor:kl-bound} \emph{\ABNC{}}.
One such example is $\fdivfun(x)= \int_1^x \sqrt{2\log(1+s)} \, \dee s$, which satisfies the conditions of \cref{thm:continuous-KL} with $c_1=2$ and $c_2=1/\sqrt{2}$ since $\fdivfun$ is increasing,
and satisfies the conditions of \cref{cor:kl-bound} with $\konst_1 =0$ and $\konst_2 = \sqrt{2}$. Thus, we obtain
\*[
  \regret{T}{\conttargetvec}
  \leq 2 \sqrt{(T+1)(1 + \KL{\conttargetvec}{\basemeasvec})} + \sqrt{8T} \, .
\]

The $\fdivfun$-divergence \FTRL{} solution with \lrname{} $\lr{t+1}$ is given by
\*[
\contrn{t+1}{\contexpidx} 
  & = \exp\cbra{\rbra{\konst^*_{t+1}-\lr{t+1}\Loss{t}{\contexpidx}}_+^2 / 2}-1,
\]
where $\konst^*_{t+1}\in\Reals$ solves $\int \sbra[1]{\exp\cbra[1]{\rbra{\konst^*_{t+1}-\lr{t+1}\Loss{t}{\contexpidx}}_+^2 / 2}-1} \basemeas{\dee \contexpidx} = 1$.
Note that this formula 
is heuristically similar to \NH{} when $\abs{\genexpset}=\numexperts$, which assigns weights to experts according to
\*[
 \weight{t}{\expidx} 
   & \propto \sbra{{\scriptstyle \sum_{s=1}^{t-1}} \inner{\lossvec{s}}{\weightvec{s}} - \Loss{t-1}{\expidx}}_+ \exp\rbra{\sbra{{\scriptstyle \sum_{s=1}^{t-1}} \inner{\lossvec{s}}{\weightvec{s}} - \Loss{t-1}{\expidx}}_+^2/2c_{t-1}},
\]
where $c_{t-1}$ 
solves $\sum_{\expidx\in\experts} \exp\rbra[1]{\sbra{{\scriptstyle \sum_{s=1}^{t-1}} \inner{\lossvec{s}}{\weightvec{s}} - \Loss{t-1}{\expidx}}_+^2/2c_{t-1}} = e \numexperts$.
\end{example}

\subsection{Lower bound for quantile regret}
\label{sec:quantile-regret-bound}

Bounds for quantile regret and KL regret can be related by observing a) that for $\orderidx{\epsidx}$ denoting a point-mass on the $\epsidx$th best expert with respect to $\Lossvec{T}$,
\*[
   \regret{T}{\orderidx{\epsidx}}
  		\leq \regret{T}{u_\quanteps},
\]
where $\quanteps = \epsidx/\numexperts$ and $u_\quanteps = \frac{1}{\epsidx}\sum_{j=1}^\epsidx\orderidx{j}$ is the uniform distribution over the top $\epsidx$ experts, and b) that $\KL{u_\quanteps}{\uniformdist(\experts)} = \log(1/\quanteps)$.
Thus, KL upper bounds are also upper bounds on quantile regret and quantile regret lower bounds are also lower bounds on certain KL regrets. 
With this in mind, we provide the first general lower bound for quantile regret on $\numexperts$ experts.
Our lower bound is matching (up to lower order terms)
the leading term in our upper bound for quantile regret 
achieved by \ABNC{} (\cref{example:root_log_regularizer}) when $\basemeasvec$ is uniform on $\experts$ and $\gentargetvec$ is uniform on only the top quantile of experts, establishing the minimax rate of quantile regret as $\sqrt{T\log(1/\quanteps)}$.

\begin{restatable}{theorem}{LowerBound}
\label{thm:lower-bound}
For all $\numexperts \in \Nats$ there exists a probability distribution $\lossdistn{}$ on $[0,1]^\numexperts$ such that for any sequence of player predictions $(\weightvec{t})_{t\in\Nats} \subseteq \simp(\experts)$, $\epsidx \in \{1,\dots,\floor{\numexperts/4}\}$, and $T\in\Nats$,
\*[
  \EE_{\lossvec{1:T}\sim\lossdistn{}^{\otimes T}} \regret{T}{\orderidx{\epsidx}}
    & \geq
    \sqrt{(T/2)\rbra{  \log\rbra[1]{1/\quanteps} -2\log2+1/\pi}} - \sqrt{2/\pi} - 2\log\numexperts - \log 2,
\]
where $\orderidx{\epsidx}$ is the point-mass on the $\epsidx$th best expert with respect to $\Lossvec{T}$ and $\quanteps = \epsidx/\numexperts$. 
\end{restatable}

\subsection{Intuition for $\fdivfun$-divergence \FTRL{} in the KL regret paradigm}
The conditions on $\fdivfun$ in \cref{thm:continuous-KL} are essentially the minimal conditions needed for the summation terms in \cref{lem:gen-ftrl-decomp} to cancel with each other regardless of the actual losses.
Intuitively, to achieve $\fdivfun$-divergence regret bounds with $\fdivfun$-divergence \FTRL{} using the bound of \cref{lem:gen-ftrl-decomp}, 
these terms must cancel so that there is no dependence in the final bound on the regularity of the distributions actually selected by the algorithm.
The condition on $\fdivfun$ in \cref{cor:kl-bound} is essentially the minimal condition needed for Jensen's inequality to imply $\fdivergence{\fdivfun}{\cdot}{\cdot} \lesssim\sqrt{\KL{\cdot}{\cdot}}$.
That it is possible to satisfy both of these conditions simultaneously is the crucial observation that enables our result.

All regularizers to which both \cref{thm:continuous-KL,cor:kl-bound} apply are essentially equivalent to $\fdivfun(x)= \int_1^x \sqrt{2\log(1+s)} \dee s$, meaning they have the same asymptotic growth rate.
To see this, first observe that this $\fdivfun$ has the minimum amount of curvature needed to satisfy $\fdivfun'' \cdot (\fdivfun+\fconstA) \geq \fconstB$ since $(x\sqrt{\log x})\cdot(x\sqrt{\log x})''$ is asymptotically constant.
Second, this $\fdivfun$ has the largest asymptotic growth rate that still satisfies $f(x)\leq \konst_1+\konst_2 x \sqrt{\log(1+x)}$.
These two facts together constrain the shape of $\fdivfun$ to the root-logarithmic choice.

\section{Semi-adversarial regret bounds}\label{sec:semiadv-main-results}

We now turn to a another perspective on prediction with expert advice for which \FTRL{} with linearly decomposable regularizers is optimal.
The semi-adversarial paradigm (introduced by \citep{semiadv}) consists of a family of constraints on the adversary's choice of loss distribution, and the goal of the player is to learn as well as possible for the true constraint without having to know the constraint in advance. More precisely, the setting is characterized by an unknown \emph{time-homogeneous convex constraint} on the adversary's choice of loss distribution, which is formally represented by a convex set of probability distributions on $\bounddrvspace$, denoted by $\distnball{}$. 
At each round $t$, the adversary is free to select any distribution from $\distnball{}$ to sample $\lossvec{t}$ from. Note that when $\distnball{}$ is the set of all probability distributions, the worst-case adversarial setting is recovered, and when $\distnball{}$ is a singleton, the stochastic setting is recovered. 
In this section, we describe a new \FTRL{} algorithm (\ABN{}) that achieves minimax optimal expected regret without requiring knowledge of $\distnball{}$ in advance.

\subsection{Adaptive minimax optimality}

Minimax regret in the semi-adversarial paradigm is quantified using a few key objects that summarize $\distnball{}$. The first is the collection of \emph{effective stochastic gaps}, defined for each expert $\expidx\in\experts$ by
\*[
  \Deltaexp{\expidx}
  = \inf_{\lossdistn{}\in\distnball{}} \max_{\expidxdum\in\experts} \EE_{\lossvec{}\sim\lossdistn{}} \sbra{\loss{}{\expidx} - \loss{}{\expidxdum}}.
\]
Using these, \citet{semiadv} define the stochastic gap
\*[
  \Deltaeff
  = \min\set{\Deltaexp{\expidx} \setdelim \expidx\in\experts,\Deltaexp{\expidx}>0}.
\]
The second object \citet{semiadv} define is the set of \emph{effective experts}, which is a subset of $\experts$ defined by
\*[
  \effexperts
  = \set{\expidx\in\experts \setdelim \Deltaexp{\expidx} = 0}.
\]
This is the set of all experts who are optimal in expectation for some element of $\distnball{}$ (or possibly in the limit along some sequence in $\distnball{}$). The number of effective experts is then $\numeffexperts = \card{\effexperts}$.

In this paradigm, the goal is to compete against 
the best expert. 
Letting $\unorderidx{\expidx}$ denote a point-mass on expert $\expidx$, for notational simplicity we set $\bestregret{T} \equiv \max_{\expidx\in\experts} \regret{T}{\unorderidx{\expidx}}$.
Further, although the player does not expect to have knowledge of $\distnball{}$ in advance, the goal is to develop methods that do as well as they possibly could have \emph{if they had access to properties of $\distnball{}$ in advance}. To characterize this, we say an algorithm (which only has knowledge of $\numexperts$ in advance) is \emph{adaptively minimax optimal} if there exists a constant $C$ such that, for all $\numexperts$ and $(\numeffexperts, \Deltaeff)$ pairs, the expected regret of the algorithm is within a factor of $C$ from the minimax regret had the algorithm had access to $(\numeffexperts, \Deltaeff)$ in advance, for sufficiently large $T$ (where sufficiently large may depend on $\numexperts$, $\numeffexperts$, and $\Deltaeff$). For a precise mathematical formulation of this concept, see Section~3.1 of \citet{semiadv}.

When $\numeffexperts=1$, Proposition~4 of \citet{mourtada2019optimality} shows that the minimax regret is of order no smaller than $(\log\numexperts)/\Deltaeff$, and when $\numeffexperts>1$, Theorem~2 of \citet{semiadv} shows that the minimax regret is of order no smaller than $\sqrt{T\log\numeffexperts}$. 
\citet{semiadv} prove that \OGHedge{} is \emph{not} adaptively minimax optimal in the semi-adversarial paradigm, and their argument applies to other similar \OGHedge{}-based algorithms, such as \acronymstyle{prod} \citep{cesabianchi2007secondorder}, \acronymstyle{AdaHedge} \citep{vanerven11adahedge}, and \acronymstyle{Adapt-ML-Prod} \citep{gaillard14}.
Further, they provide an algorithm 
that achieves 
\[\label{eqn:bnr20-upper-bound}
	\EE \bestregret{T}
	\lesssim \sqrt{T \log \numeffexperts} + \ind{\numeffexperts=1} \frac{\log\numexperts}{\Deltaeff} +  \ind{\numeffexperts>1} \frac{(\log\numexperts)^{3/2}}{\Deltaeff}.
\]
Since $\ind{\numeffexperts>1} (\log\numexperts)^{3/2}\Deltaeff^{-1}$ is lower order when $\numeffexperts>1$, 
this algorithm
is adaptively minimax optimal. 
We now present a new algorithm, \ABN{}, which is also adaptively minimax optimal \emph{and} achieves a better regret bound for small $T$.
\subsection{Semi-adversarial regret bound for \ABN{}}

Let $\partialentropyB:\cinter{0,1}\to \Reals$ be given by
\*[
  \partialentropyB(x)
    & = \begin{cases}
      x\sqrt{2 \log(1/x)} - \sqrt{\frac{\pi}{2}} \erf\rbra{\sqrt{\log(1/x)}} + x (\numexperts - 1)\sqrt{\frac{\pi}{2}} & x\in(0,1]\\
      -\sqrt{\pi/2} & x=0,
  \end{cases}
\]
set $\entropyB(\weightvec{}) = \sum_{\expidx\in\experts} \partialentropyB(\weight{}{\expidx})$, and define \ABN{} to be \FTRL{} with $\basemeasvec$ defined as counting measure, regularizer 
$\gendiv{-\partialentropyB}{\basemeasvec}{\weightvec{}} = -\entropyB(\weightvec{})$, and \lrname{} $\lr{t} = 2/\sqrt{t}$. 
Note that this corresponds to \FTRL{} with a linearly decomposable regularizer, and an intuitive explanation for this choice of regularizer can be found in \cref{sec:intuition-abn}.
We then have the following regret bound for \ABN{}, which removes the term $\ind{\numeffexperts>1} (\log\numexperts)^{3/2}\Deltaeff^{-1}$ from \cref{eqn:bnr20-upper-bound}.

\begin{corollary}\label{fact:simple-carl-bound}
For any time-homogeneous convex constraint $\distnball{}$, \ABN{} achieves:\\
For all $T$,
\*[
	\EE \bestregret{T} \leq \sqrt{2T\log\numexperts},
\]
and if $T > 8(\log\numeffexperts)\Deltaeff^{-2}$,
\*[
	\EE \bestregret{T} \leq \sqrt{2T\log\numeffexperts} + 25\frac{\log\numexperts}{\Deltaeff}.
\]
\end{corollary}

\cref{fact:simple-carl-bound} follows from \cref{thm:discrete-semiadv} in \cref{sec:proof-discrete-semidav}, which is a more refined regret bound.

\subsection{Intuition for \ABN{} in the semi-adversarial paradigm}
\label{sec:intuition-abn}

The \ABNreg{} regularizer can be motivated by the following intuition from the \OGHedge{} algorithm.
For \iid{} losses, the upper bounds for \OGHedge{} \citep[Theorem~2,][]{mourtada2019optimality} are only optimal with \lrname{} $\lr{t}\gtrsim \sqrt{(\log\numexperts) / t}$, and matching lower bounds for the adversarial setting suggest this \lrname{} constraint is actually necessary for optimal performance. 
Such a \lrname{} ensures that the weights of each suboptimal (in expectation) expert decay fast enough.
In the semi-adversarial paradigm with more than one effective expert, \FTRLCARE{} (of \citep{semiadv}) can be interpreted as \OGHedge{} with an adaptive \lrname{} that asymptotically satisfies $\lr{t}\gtrsim \sqrt{(\log\numeffexperts) / t}$. 
This smaller \lrname{} applied to the effective experts is necessary to incur asymptotically $\sqrt{T\log\numeffexperts}$ regret.
However, since this \lrname{} is smaller, when there are two or more effective experts the weights assigned to the ineffective experts seemingly are slightly too large. 

To rectify this, heuristically, it would be ideal to have expert-specific \lrname{}s of size $\sqrt{(\log\numeffexperts) / t}$ for the effective experts and $\sqrt{(\log\numexperts) / t}$ for the ineffective experts. Since the effective experts will have weights on the order of $1/\numeffexperts$ and the ineffective experts will have weights smaller than $1/\numexperts$, the expert-specific \lrname{} $\lr{t}(i) = \const\sqrt{\log(1/\weight{t}{\expidx}) / t}$ may plausibly achieve the desired behaviour.

The weights of \OGHedge{} are defined by the equation 
$\log(1/\weight{t}{\expidx}) = \lr{t}( \Loss{t-1}{\expidx}+\lambda_t)$, where $\lambda_t$ is chosen to ensure the weights are normalized. Replacing $\lr{t}$ with our heuristic yields $\log(1/\weight{t}{\expidx}) = \const \sqrt{\log(1/\weight{t}{\expidx}) / t}\  (\Loss{t-1}{\expidx}+\lambda_t)$, which can be rearranged to obtain
\*[
  \sqrt{\log(1/\weight{t}{\expidx})}
    & = \tfrac{c}{\sqrt{t}}(\Loss{t-1}{\expidx}+\lambda_t).
\]
By \cref{lem:gen-ftrl-soln-formula}, this is exactly the formula for the weights produced with \FTRL{} for the regularizer $-\entropyB$ and \lrname{} $\lr{t}=c/\sqrt{t}$.

As \cref{thm:discrete-semiadv} shows,
our modification to the \OGHedge{} algorithm is sufficient to yield semi-adversarial regret bounds with expected regret contribution of size $(\log\numexperts)/\Deltaeff$ from the ineffective experts, improving on the order of the regret bound for \FTRLCARE{} 
when $\numeffexperts>1$.

\section{FTRL analysis for general expert spaces}
\label{sec:ftrl-gen}

Finally,
we analyze the general performance of \FTRL{} with a linearly decomposable regularizer.
First, \cref{lem:gen-ftrl-soln-formula} provides a closed-form expression for the \FTRL{} solution on general spaces.
Specifically, it reduces solving the \FTRL{} optimization problem of \cref{eq:diagonal-regularizer} to finding a root of a one-dimensional equation (\cref{eqn:gen-normalization-formula}), rather than solving a complex, possibly infinite dimensional, optimization problem.
Second, \cref{lem:gen-ftrl-decomp} provides our fundamental regret bound for \FTRL{} with such regularizers, which we have used to obtain the results in \cref{sec:quantile-main-results,sec:semiadv-main-results}. It generalizes the well-known results involving local norms from the finite dimensional case \citep{Abernethy2009BeatingTA,zimmert19,orabona2019}, and retains an additional summation term (usually uniformly bounded in \FTRL{} analyses) that is crucial both for tight root-KL regret bounds and for tight bounds in the semi-adversarial paradigm.
All results in this section are proved in \cref{sec:proof-ftrl-decomp}.
For the remainder of this section, let $\basemeasvec\in\finitemeas\genexpspace$ be fixed and arbitrary.

\subsection{Computing \FTRL{} with linearly decomposable regularizers}

For continuously differentiable $\fdivfun:\domfdivfun(\basemeasvec)\to\Reals$,
let $\genlowerdiv = \inf_{x\in \domfdivfun(\basemeasvec)} \gendivfun'(x)$, $\genupperdiv = \sup_{x\in \domfdivfun(\basemeasvec)} \gendivfun'(x)$,
and $\gentruncfun(y) = \max(\min(y,\genupperdiv),\genlowerdiv)$. 
We focus on the setting where $\gendivfun'$ is strictly increasing, and thus $\gentruncfun(y)$ truncates its argument to the domain of $[\gendivfun']^{-1}$.

\begin{restatable}{lemma}{FTRLsolutionFormula}
\label{lem:gen-ftrl-soln-formula}
Suppose $\gendivfun :\domfdivfun(\basemeasvec) \to \Reals$ is twice continuously differentiable with $\gendivfun''>0$ on $\interior(\domfdivfun(\basemeasvec))$ and $\gendivfun''(0+)>0$.
For any $\Lossvec{} \in \boundrvspace$ and $\lr{} > 0$,
\*[
	\genrnvec{}^*(\theta) = [\gendivfun']^{-1}\rbra[1]{ \gentruncfun(- \lr{} \Loss{}{\genexpidx} + \konst^*) }
\] 
satisfies $\genrnvec{}^* \in \argmin_{\genrnvec{} \in \probrvspace(\basemeasvec)} \cbra[0]{\innermeas[0]{\Lossvec{}}{\genrnvec{}}{\basemeasvec} + \lr{}^{-1} \gendiv{\gendivfun}{\basemeasvec}{\genrnvec{}}}$,
where $\konst^*\in\Reals$ solves
\[
\label{eqn:gen-normalization-formula}
	\int [\gendivfun']^{-1}\rbra{\gentruncfun(- \lr{} \Loss{}{\genexpidx} + \konst^*)} \basemeas{\dee \genexpidx} = 1.
\]
Further, this solution is unique 
up to modification on a set of $\basemeasvec$-measure $0$.
\end{restatable}

For finite expert classes, solving the normalizing equation provided by \cref{lem:gen-ftrl-soln-formula} to a given precision is essentially the same difficulty as normalizing the weights for the classical \OGHedge{} algorithm. For example, for \ABN{} the range of the normalizing constant scales with $\log\numexperts$, and solving this using the bisection method would require only $\Oo(\log\log\numexperts)$ times more computation than normalizing \OGHedge{}. For uncountable expert classes, many of the general algorithms for prediction with expert advice from previous work are not easily modified to apply, and the standard analyses of regret bounds rely heavily on a finite $\numexperts$. In this setting, solving the \FTRL{} normalizing equation is essentially as difficult as normalizing the posterior distribution for Bayesian inference. That is, often computationally intractable, yet regularly studied for its theoretical properties. Extending approximation techniques for Bayesian inference to approximate the solution to this optimization problem is an interesting question for future work.

\subsection{Choosing $\fdivfun$ to obtain specific regret bounds}

In order to state our generic decomposition for the regret of \FTRL{} algorithms on abstract spaces, we rely on the definition of the one-dimensional \emph{Bregman divergence}, which is defined for any continuously differentiable, strictly convex $\fdivfun: \Reals \to \Reals$ by
\*[
	\bregman{\fdivfun}{x}{y}
	= \fdivfun(x) - \fdivfun(y) - \fdivfun'(y) (x-y).
\] 
\begin{restatable}{theorem}{FTRLdecomp}
\label{lem:gen-ftrl-decomp}
Suppose $\gendivfun :\domfdivfun(\basemeasvec) \to \Reals$ is twice continuously differentiable with $\gendivfun''>0$ on $\interior(\domfdivfun(\basemeasvec))$ and $\gendivfun''(0+)>0$. 
For all sequences $(\lossvec{t})_{t\in\Nats}\subseteq\boundrvspace$ and $(\mean{t})_{t\in\Nats}\subseteq\Reals$,
there exist $\intgenrnvec{t+1}\in\convhull(\genrnvec{t},\genrnvec{t+1})$ and $\inttgenrnvec{t+1}\in\convhull(\genrnvec{t},\orthgenrnvec{t+1})$ 
such that for all $\gentargetvec{}\in\probmeasll{\basemeasvec}\genexpspace$, $(\intgenrnbothvec{t+1})_{t\in\Nats}\in\set{(\intgenrnvec{t+1})_{t\in\Nats},(\inttgenrnvec{t+1})_{t\in\Nats}}$, and $T$, $\gendivnoarg{\gendivfun}{\basemeasvec}$-regularized \FTRL{} achieves
\[\label{eqn:gen-expansion-A}
	\regret{T}{\gentargetvec{}}
	&\leq \frac{1}{\lr{T+1}}\gendivrn{\gendivfun}{\basemeasvec}{\rnderiv{\gentargetvec{}}{\basemeasvec}} - \frac{1}{\lr{1}} \min_{\genrnvec{}\in\probrvspace(\basemeasvec)}\gendiv{\gendivfun}{\basemeasvec}{\genrnvec{}} \\
	&\qquad + \sum_{t=1}^T \sbra{\int \frac{\lr{t}}{2} \frac{\rbra{\loss{t}{\genexpidx} - \mean{t}}^2}{\gendivfun''(\intgenrnboth{t+1}{\genexpidx})} \basemeas{\dee\genexpidx} - \rbra{\frac{1}{\lr{t+1}} - \frac{1}{\lr{t}}} \gendiv{\gendivfun}{\basemeasvec}{\genrnvec{t+1}}},
\]
where
\*[
	\orthgenrnvec{t+1} \in \argmin_{\genrnvec{} \in \posrvspace(\basemeasvec)} \cbra[2]{\innermeas[0]{\lossvec{t} - \mean{t}\mathbbm{1}}{\genrnvec{}}{\basemeasvec}
    + \frac{1}{\lr{t}} \int \bregman{\gendivfun}{\genrn{}{\genexpidx}}{\genrn{t}{\genexpidx}} \basemeas{\dee\genexpidx}}.
\]

Further, if $\inf_{\genexpidx\in\genexpset} \loss{t}{\genexpidx} \geq \mean{t}$, then $\inttgenrnvec{t+1} \leq \genrnvec{t}$ pointwise.
\end{restatable}

\cref{lem:gen-ftrl-decomp} is a general result that can be used to prove regret bounds for a variety of settings under the general \FTRL{} framework we describe, as we have already done in \cref{sec:quantile-main-results,sec:semiadv-main-results}. Furthermore, we highlight that the functional form of \cref{lem:gen-ftrl-decomp} makes it clear how to select the regularizers for both the quantile regret and semi-adversarial paradigms. Specifically, in the former, the terms in summation in \cref{eqn:gen-expansion-A} are balanced to cancel, while in the latter, the terms in summation in \cref{eqn:gen-expansion-A} are balanced to contribute the same order to the regret. These two cases correspond to the relationship 
\[\label{eq:balanced-as-all-regularizers-should-be}
  \gendivfun\cdot\gendivfun'' = \pm 1.
\]

Since the generic bound we obtain in \cref{lem:gen-ftrl-decomp} yields components based upon integrals of both $\gendivfun$ and $1/\gendivfun''$, balancing these terms without free tuning parameters requires \cref{eq:balanced-as-all-regularizers-should-be} to approximately hold. Heuristically, trying to balance these terms with tuning parameters rather than by the choice of $\gendivfun$ seems to lead to non-adaptive or non-uniform bounds. \OGHedge{} provides an example of this for both the semi-adversarial and KL regret cases. In the semi-adversarial paradigm, \citet{semiadv} showed that \OGHedge{} cannot be tuned in a way that is minimax optimal and agnostic to the semi-adversarial constraint that prevails. Similarly, without tuning the \lrname{} to be dependent on the comparator distribution (equivalently, the quantile of interest), KL (and quantile) regret bounds for \OGHedge{} are suboptimal. 

\subsection{Applications of root-KL regret bounds for continuous experts}

We now briefly discuss two applications that highlight the immediate benefits of our general analysis (and consequently results of \cref{sec:quantile-main-results}) applying beyond finite expert spaces.

\paragraph{Predicting as well as the terminal posterior using $\fdivfun$-divergence \FTRL{}}
A reasonable choice of distribution to measure regret against is the posterior distribution $\hat\genpriorvec_{T}$ after having seen $T$ rounds of data. 
A consequence of \cref{thm:continuous-KL,cor:kl-bound} is that, for bounded log-likelihoods, the total loss incurred by making predictions according to $\fdivfun$-divergence \FTRL{} for suitable $\fdivfun$ is bounded by 
the loss incurred by the terminal posterior $\hat\genpriorvec_{T}$ plus an excess regret of the order $\sqrt{T \, \KL{\hat\genpriorvec_{T}}{\genpriorvec}}$. 
This excess loss is smaller than $T+\KL{\hat\genpriorvec_{T}}{\genpriorvec}$, which is the best available bound for excess loss when predicting according to the posterior $\hat\genpriorvec_{t}$ at each round $t\in\range{T}$, 
and smaller than $\sqrt{T}\ \KL{\hat\genpriorvec_{T}}{\genpriorvec}$, which is the best available bound for the excess loss when predicting with \OGHedge{} if it is not tuned with \emph{a priori} knowledge of $\KL{\hat\genpriorvec_{T}}{\genpriorvec}$.
These worse excess loss bounds follow from the analysis contained in \citet{zhang06information}, although he only compares against the ``true'' data-generating parameter.

\paragraph{Model selection using $\fdivfun$-divergence \FTRL{}}
Extending the interpretation of \citet{orabona2016} by \citet{pmlr-v125-foster20a} to countable union model classes, we can consider an infinite sequence of disjoint finite expert classes (note any nested sequence can be made disjoint), $\genexpset_1, \genexpset_2,\dots$, and let $\genexpset = \Union_{m\geq 1}\genexpset_m$.
Assigning to $\genexpset$ the prior $\genprior{\genexpidx} \propto \frac{1}{m^2 \card{\genexpset_m}}$ for each $\genexpidx\in\genexpset_m$, we recover the following regret bound for $\fdivfun$-divergence \FTRL{} with \lrname{} $\lr{t} \propto 1/\sqrt{t}$ from \cref{cor:kl-bound}:
\*[
	\regret{T}{\unorderidx{\genexpidx}} \leq \Oo\rbra{\sqrt{T(\log\card{\genexpset_m}+\log m)}}
	& \text{ for all } \genexpidx\in\genexpset_m \text{ and all } m\in\Nats.
\]

\section{Limitations}\label{sec:discussion}
One limitation of our results, which is common in the online learning literature, is that they only apply to bounded losses. 
The extension of learning theory to unbounded losses has seen increased interest in recent years \citep{grunwald2020fastrates,alquier20unbounded,mourtada21unbounded}, although it remains a major open problem to achieve guarantees for arbitrary losses. In particular, log loss for many nonparametric learners is not covered by the current unbounded loss literature, and will have important implications 
in statistical learning and density estimation 
when resolved.

A second limitation of our work is that finding the implicit normalizing constant---$\konst^*$ in \cref{lem:gen-ftrl-soln-formula}---is computationally difficult, as was also observed by \citet{alquier20unbounded}. It is at least as costly as finding the normalizing constant for a Bayesian posterior (or running \OGHedge{}), which corresponds to one evaluation of the left-hand side of \cref{eqn:gen-normalization-formula}. For \OGHedge{}, where $\gendivfun(x) = x\log(x)$, one evaluation of that equation is sufficient, but in general finding the root of \cref{eqn:gen-normalization-formula} to a fixed precision will take a number of evaluations depending on the desired precision and the range of possible values, and so \FTRL{} with a general $\gendivfun$ will be that marginally more computationally expensive than \OGHedge{}. Variational approaches for Bayesian inference, which avoid direct normalization of the posterior, may also be applicable for \FTRL{} with a general $\gendivfun$, and that line of inquiry may lead to novel, efficient, and high-performance learning algorithms.

A final limitation is that we do not obtain variance bounds together with uniform root-KL bounds, although our approach leads to both separately (\cref{fact:gen-variance-bound} and \cref{cor:kl-bound}), so we are optimistic our techniques can lead to such bounds, which would resolve multiple open problems. 

\begin{ack}

JN is supported by an NSERC Vanier Canada Graduate Scholarship and the Vector Institute. 
BB is supported by an NSERC Canada Graduate Scholarship and the Vector Institute. 
FO is partly supported by the National Science Foundation under grants no. 1908111 "AF: Small: Collaborative Research: New Representations for Learning Algorithms and Secure Computation" and no. 2046096 "CAREER: Parameter-free Optimization Algorithms for Machine Learning". 
DMR is supported in part by an NSERC Discovery Grant, an Ontario Early Researcher Award, and a stipend provided by the Charles Simonyi Endowment. 
We thank Mufan Li and Mahdi Haghifam for helpful feedback on early drafts.
\end{ack}

\printbibliography

\neurips{\input{checklist}} 

\neurips{\newpage} 
\appendix
\section{Proofs for \FTRL{} analysis for general expert spaces}
\label{sec:proof-ftrl-decomp}

We first prove the formula for the \FTRL{} solution prescribed by \cref{lem:gen-ftrl-soln-formula}, which we recall next.

\FTRLsolutionFormula*

In the proof below, we treat the formula for the optimal solution as an educated guess (inspired by the finite-dimensional case using Lagrange multipliers). The proof then verifies that the proposed solution is well-defined and in fact achieves the optimum. It is also possible to derive the optimal solution directly using a version of the Lagrange multiplier method for convex optimization problems on Banach spaces with cone constraints (see, for example, \citep{luenberger1969optimization}). However, to apply such a result, one must still verify the existence of Lagrange multipliers. For example, existence of the Lagrange multiplier for the constraint ``$\genrnvec{}$ integrates to $1$'' exactly corresponds to existence of a solution to \cref{eqn:gen-normalization-formula}. Thus it is not significantly more or less laborious to take such an approach over the ``guess and check'' method we have employed.

\begin{proof}[Proof of \cref{lem:gen-ftrl-soln-formula}]
Let $\Losslower = \inf_{\genexpidx\in\genexpset}\Loss{}{\genexpidx}$ and $\Lossupper = \sup_{\genexpidx\in\genexpset}\Loss{}{\genexpidx}$.
Since $\gendivfun''>0$, 
$\gendivfun'$ is strictly increasing and restricted to a non-negative domain, so $[\gendivfun']^{-1}$ exists, is strictly increasing, and is non-negative. 
Let $a = \gendivfun'(1 / \basemeas{\genexpset})$, and note that since $1/\basemeas{\genexpset} \in \domfdivfun(\basemeasvec{})$, $\gentruncfun(a)=a$.
Let $\gennormfun:\Reals\to \basemeas{\genexpset}\cdot\domfdivfun(\basemeasvec{}) $ be given by
\*[
	\gennormfun(\konst) = \int [\gendivfun']^{-1}\rbra{\gentruncfun(- \lr{} \Loss{}{\genexpidx} + \konst)} \basemeas{\dee \genexpidx}.
\]

First, note that for all $\konst$, $\gennormfun(\konst) \leq \basemeas{\genexpset} \sup_{\genexpidx\in\genexpset} [\gendivfun']^{-1}\rbra{\gentruncfun(- \lr{} \Loss{}{\genexpidx} + \konst)}$.
So, since $\Lossvec{} \geq \Losslower$, $\gennormfun(\konst) \leq \basemeas{\genexpset} [\gendivfun']^{-1}\rbra{\gentruncfun(- \lr{} \Losslower + \konst)}$. Thus,
\[\label{eqn:gen-ftrl-soln-1}
	1
	&=
	\int [\gendivfun']^{-1}(a) \basemeas{\dee \genexpidx} \\
	&\leq
	\int [\gendivfun']^{-1}(\gentruncfun(-\lr{} \Loss{}{\genexpidx} + \lr{}\Lossupper + a)) \basemeas{\dee \genexpidx} \\
	&= \gennormfun(\lr{}\Lossupper + a) \\
	&\leq \basemeas{\genexpset} [\gendivfun']^{-1}(\gentruncfun(\lr{}(\Lossupper-\Losslower) + a)),
\]
where the second inequality follows since $\Lossvec{} \leq \Lossupper$.
Further, non-negativity of $[\gendivfun']^{-1}$ gives
\[\label{eqn:gen-ftrl-soln-2}
	0 &\leq \gennormfun(\lr{}\Losslower + a) \\
	& = \int [\gendivfun']^{-1}(\gentruncfun(\lr{}(\Losslower - \Loss{}{\genexpidx}) + a)) \basemeas{\dee \genexpidx} \\
	&\leq \int [\gendivfun']^{-1}(a) \basemeas{\dee \genexpidx}  \\
	& = 1.
\]

By Leibniz rule and the inverse function theorem, \cref{eqn:gen-ftrl-soln-1,eqn:gen-ftrl-soln-2} imply that, for all $\konst\in \cinter{\lr{}\Losslower + a, \lr{}\Lossupper+a}$,
\*[
  \gennormfun'(\konst)
  & = \int \frac{\ind{\genupperdiv\geq- \lr{}\Loss{}{\genexpidx} + \konst \geq \genlowerdiv}} {\gendivfun''\circ[\gendivfun']^{-1}\rbra{ \gentruncfun(- \lr{}\Loss{}{\genexpidx} + \konst})} \basemeas{\dee \genexpidx} \\
  & \leq \frac{\basemeas{\genexpset}}{\inf_{\genexpidx \in \ointer{0, [\gendivfun']^{-1}(\gentruncfun(\lr{}(\Lossupper-\Losslower) + a))}} \gendivfun''(\genexpidx)} \\
  & <\infty.
\]
Hence, $g$ is (Lipschitz) continuous on this interval, and then, by the intermediate value theorem, there exists $\konst^*$ solving \cref{eqn:gen-normalization-formula}.

Let $\ftrlobj{}:\posrvspace(\basemeasvec)\to\Reals$ be given by
\*[
	\ftrlobj{}(\genrnvec{})
	  & = \innermeas[0]{\Lossvec{}}{\genrnvec{}}{\basemeasvec} + \frac{1}{\lr{}} \gendiv{\gendivfun}{\basemeasvec}{\genrnvec{}}.
\]

To avoid notational clutter, for the remainder of the proof all statements involving $\genexpidx$ implicitly hold only $\basemeasvec$-a.s.
Let $\genrn{}{\genexpidx} = [\gendivfun']^{-1}\rbra{\gentruncfun( - \lr{} \Loss{}{\genexpidx} + \konst^*)}$ and note that the definition of $\konst^*$ implies that $\int \genrn{}{\genexpidx} \basemeas{\dee \genexpidx} = 1$. We now argue that $-\lr{} \Lossvec{} + \konst^* \leq \genupperdiv$. 

If $\basemeasmin = 0$, then $\genupperdiv = \infty$, so trivially $-\lr{} \Lossvec{} + \konst^* \leq \genupperdiv$. 

If $\basemeasmin \neq 0$, then $\basemeasvec$ is purely atomic and, since $\fdivfun'$ is increasing, 
$\genupperdiv = \fdivfun'(1/\basemeasmin)$. 
By definition of $\gentruncfun$, $\genrnvec{} \leq 1/\basemeasmin$. 
If $\genrnvec{} < 1/\basemeasmin$, then $\gentruncfun( - \lr{} \Loss{}{\genexpidx} + \konst^*) < \genupperdiv$, so $- \lr{} \Loss{}{\genexpidx} + \konst^* \leq \genupperdiv$. Otherwise, there exists an atom $A\in\genexpSF$ for $\basemeasvec{}$ such that $\genrn{}{\genexpidx} = 1/\basemeasmin$ for $\genexpidx\in A$.
Since $\genrnvec{}\in\probrvspace(\basemeasvec)$, we must have $\genrn{}{\genexpidx} = \ind{\genexpidx\in A}/\basemeasmin$; equivalently $\genrnvec{}$ is the Radon--Nikodym derivative of a single-atom probability measure completely concentrated on $A$. 
Without loss of generality, $A= \set{\genexpidx\setdelim \genrn{}{\genexpidx} = 1/\basemeasmin}$ and
\*[
	\fdivfun'(\genrn{}{\genexpidx})
	= \genlowerdiv + (\genupperdiv-\genlowerdiv)\ind{\genexpidx\in A}.
\]
Then, any $\konst^*$ that satisfies
\*[
	\begin{cases}
	-\lr{} \Loss{}{\genexpidx} + \konst^* \geq \genupperdiv,
	& \genexpidx \in A \\
	-\lr{} \Loss{}{\genexpidx} + \konst^* \leq \genlowerdiv,
	& 
	\genexpidx \not\in A
	\end{cases}
\]
is valid and gives the same solution. 
In particular, since there is some $\konst^*$ satisfying this, then $\tilde{\konst}^* = \genupperdiv + \lr{} [\basemeasvec\text{-}\esssup]_{\genexpidx\in A}\Loss{}{\genexpidx} \leq \konst^* \leq \genlowerdiv + \lr{} [\basemeasvec\text{-}\essinf]_{\genexpidx\not\in A}\Loss{}{\genexpidx}$, 
which implies $\tilde{\konst}^*$ is a valid solution. 
Finally, we claim that $\Lossvec{}$ must be constant on $A$.
To see this, suppose otherwise there is some $\tildeLossvec{}$ such that for $\genexpidx\in\tilde A \subseteq A$, $\Loss{}{\genexpidx} < \tildeLossvec{}$, and for $\genexpidx\in A \setminus \tilde A$, $\Loss{}{\genexpidx} \geq \tildeLossvec{}$. Then the inverse images of $(\infty, \tildeLossvec{})$ and $[\tildeLossvec{}, \infty)$ must partition $A$ since $\Lossvec{}$ is measurable, but one of them must have $\basemeasvec$ measure zero since $A$ is an atom.
Thus, it is also true that $\tilde{\konst}^* = \genupperdiv + \lr{} [\basemeasvec\text{-}\essinf]_{\genexpidx\in A}\Loss{}{\genexpidx}$,
so we have argued that, without loss of generality, in all cases $-\lr{} \Lossvec{} + \konst^* \leq \genupperdiv$.

With this in mind, consider any $\genrndumvec{}\in\probrvspace(\basemeasvec)$.
Since $\fdivfun$ is continuously differentiable and convex,
\[\label{eqn:gen-ftrl-soln-3}
	\ftrlobj{}(\genrndumvec{})
	& = \ftrlobj{}(\genrnvec{})
	  + \innermeas[0]{\Lossvec{}}{\genrndumvec{}-\genrnvec{}}{\basemeasvec}
	  + \frac{1}{\lr{}} \int \sbra{\gendivfun(\genrndumvec{}(\genexpidx)) - \gendivfun(\genrnvec{}(\genexpidx))} \basemeas{\dee \genexpidx} \\
	& \geq \ftrlobj{}(\genrnvec{})
	  + \innermeas[0]{\Lossvec{}}{\genrndumvec{}-\genrnvec{}}{\basemeasvec}
	  + \frac{1}{\lr{}} \int \gendivfun'(\genrn{}{\genexpidx}) \sbra{\genrndum{}{\genexpidx} - \genrn{}{\genexpidx}} \basemeas{\dee \genexpidx} \\
	&= \ftrlobj{}(\genrnvec{})
	  + \int_{\genrn{}{\genexpidx}>0} \Loss{}{\genexpidx} \sbra{\genrndum{}{\genexpidx} - \genrn{}{\genexpidx}} \basemeas{\dee \genexpidx}
	  + \int_{\genrn{}{\genexpidx}=0} \Loss{}{\genexpidx} \sbra{\genrndum{}{\genexpidx} - \genrn{}{\genexpidx}} \basemeas{\dee \genexpidx} \\
	  &\qquad 
	  + \frac{1}{\lr{}} \int_{\genrn{}{\genexpidx}>0} \gendivfun'(\genrn{}{\genexpidx}) \sbra{\genrndum{}{\genexpidx} - \genrn{}{\genexpidx}} \basemeas{\dee \genexpidx} 
	  + \frac{1}{\lr{}} \int_{\genrn{}{\genexpidx}=0} \gendivfun'(\genrn{}{\genexpidx}) \sbra{\genrndum{}{\genexpidx} - \genrn{}{\genexpidx}} \basemeas{\dee \genexpidx} \\
	&= \ftrlobj{}(\genrnvec{})
	  + \int_{\genrn{}{\genexpidx}>0} \Loss{}{\genexpidx} \sbra{\genrndum{}{\genexpidx} - \genrn{}{\genexpidx}} \basemeas{\dee \genexpidx}
	  + \int_{\genrn{}{\genexpidx}=0} \Loss{}{\genexpidx} \sbra{\genrndum{}{\genexpidx} - \genrn{}{\genexpidx}} \basemeas{\dee \genexpidx} \\
	  &\qquad 
	  + \frac{1}{\lr{}} \int_{\genrn{}{\genexpidx}>0} \rbra{-\lr{} \Loss{}{\genexpidx} + \konst^*} \sbra{\genrndum{}{\genexpidx} - \genrn{}{\genexpidx}} \basemeas{\dee \genexpidx} 
	  + \frac{1}{\lr{}} \int_{\genrn{}{\genexpidx}=0} \genlowerdiv \sbra{\genrndum{}{\genexpidx} - \genrn{}{\genexpidx}} \basemeas{\dee \genexpidx} \\	
	&= \ftrlobj{}(\genrnvec{})
	  + \int_{\genrn{}{\genexpidx}>0} \frac{\konst^*}{\lr{}} \sbra{\genrndum{}{\genexpidx} - \genrn{}{\genexpidx}} \basemeas{\dee \genexpidx}
	  + \int_{\genrn{}{\genexpidx}=0} \rbra{\Loss{}{\genexpidx} +\frac{\genlowerdiv}{\lr{}}} \sbra{\genrndum{}{\genexpidx} - \genrn{}{\genexpidx}} \basemeas{\dee \genexpidx} \\
	&\geq \ftrlobj{}(\genrnvec{})
	  + \int_{\genrn{}{\genexpidx}>0} \frac{\konst^*}{\lr{}} \sbra{\genrndum{}{\genexpidx} - \genrn{}{\genexpidx}} \basemeas{\dee \genexpidx}
	  + \int_{\genrn{}{\genexpidx}=0} \frac{\konst^*}{\lr{}} \sbra{\genrndum{}{\genexpidx} - \genrn{}{\genexpidx}} \basemeas{\dee \genexpidx} \\
	&= \ftrlobj{}(\genrnvec{}),	
\]
where in the second last step we have used that if $\genrn{}{\genexpidx}=0$ then $- \lr{} \Loss{}{\genexpidx} + \konst^* \leq \genlowerdiv$ and $\genrndum{}{\genexpidx} - \genrn{}{\genexpidx} \geq 0$. 
Thus, $\genrnvec{}$ is a solution to the \FTRL{} equation. Further, since $\gendivfun$ is strictly convex, equality can only hold in \cref{eqn:gen-ftrl-soln-3} if $\basemeas{\genrnvec{} = \genrndumvec{}}=1$.
\end{proof}

To prove \cref{lem:gen-ftrl-decomp}, we need the analogue of the finite-dimensional first-order optimality condition. For any $\convexset \subseteq \rvspace(\basemeasvec)$ and any $\ftrlobj{}:\convexset\to \Reals$
and $\genrnvec{} \in \rvspace(\basemeasvec)$, the Gateaux derivative (in the direction of $\genrndumvec{}\in\rvspace(\basemeasvec)$) is
\*[
	\funcdiff{\ftrlobj{}}{\genrnvec{}}{\genrndumvec{}}
	= \lim_{\alpha \to 0}
	\frac{\ftrlobj{}(\genrnvec{} + \alpha \genrndumvec{}) - \ftrlobj{}(\genrnvec{})}{\alpha}.
\]

The following result is straightforward: we include a proof for completeness.
\begin{lemma}\label{lem:gen-first-order-optimality}
If $\convexset \subseteq \rvspace(\basemeasvec)$ is convex and $\genrnvec{} = \argmin_{\genrndumvec{}\in\convexset} \ftrlobj{}(\genrndumvec{})$, then $\funcdiff{\ftrlobj{}}{\genrnvec{}}{\genrndumvec{}-\genrnvec{}} \geq 0$ for all $\genrndumvec{}\in\convexset$ where the limit exists.
\end{lemma}

\begin{proof}[Proof of \cref{lem:gen-first-order-optimality}]
Towards a contradiction, suppose there exists $\genrndumvec{}\in\convexset$ with $\funcdiff{\ftrlobj{}}{\genrnvec{}}{\genrndumvec{}-\genrnvec{}} < 0$. Define $\genrndumvec{\alpha} = \alpha \genrndumvec{} + (1-\alpha)\genrnvec{}$ for all $\alpha\in[0,1]$. 
By definition,
\*[
	\funcdiff{\ftrlobj{}}{\genrnvec{}}{\genrndumvec{}-\genrnvec{}}
	= \lim_{\alpha \to 0}
	\frac{\ftrlobj{}(\genrnvec{} + \alpha (\genrndumvec{}-\genrnvec{})) - \ftrlobj{}(\genrnvec{})}{\alpha}
	= \lim_{\alpha \to 0} \frac{\ftrlobj{}(\genrndumvec{\alpha}) - \ftrlobj{}(\genrnvec{})}{\alpha}. 
\]
By assumption, this implies that for some $\alpha>0$, $\ftrlobj{}(\genrndumvec{\alpha}) < \ftrlobj{}(\genrnvec{})$. However, since $\genrndumvec{\alpha} \in \convexset$ for all $\alpha \in [0,1]$ by convexity, this contradicts the optimality of $\genrnvec{}$.
\end{proof}

Next, define the functional Bregman divergence for any convex $\convexset \subseteq \rvspace(\basemeasvec)$ and any $\ftrlobj{}:\convexset\to \Reals$ and $\genrnvec{}, \genrndumvec{}\in\rvspace(\basemeasvec)$ by
\*[
	\bregman{\ftrlobj{}}{\genrnvec{}}{\genrndumvec{}}
	= \ftrlobj{}(\genrnvec{}) - \ftrlobj{}(\genrndumvec{})
	- \funcdiff{\ftrlobj{}}{\genrndumvec{}}{\genrnvec{}-\genrndumvec{}}.
\]

We are now able to prove our generic \FTRL{} bound, \cref{lem:gen-ftrl-decomp}, which we recall here for completeness.

\FTRLdecomp*

\begin{proof}[Proof of \cref{lem:gen-ftrl-decomp}]
For all $t$, let
\*[
	\ftrlobj{t}(\genrnvec{})
	  & = \innermeas[0]{\Lossvec{t-1}}{\genrnvec{}}{\basemeasvec} + \frac{1}{\lr{t}} \gendiv{\gendivfun}{\basemeasvec}{\genrnvec{}}.
\]
By the trivial extension of \citet[Lemma~7.1]{orabona2019} beyond $\Reals^d$, we have
\*[
	\regret{T}{\convextargetelem}
	&\leq \frac{1}{\lr{T+1}}\gendivrn{\gendivfun}{\basemeasvec}{\rnderiv{\gentargetvec{}}{\basemeasvec}} - \frac{1}{\lr{1}} \min_{\genrnvec{}\in\probrvspace(\basemeasvec)}\gendiv{\gendivfun}{\basemeasvec}{\genrnvec{}} + \sum_{t=1}^T \sbra{\ftrlobj{t}(\genrnvec{t}) - \ftrlobj{t+1}(\genrnvec{t+1}) +\inner[0]{\lossvec{t}}{\genrnvec{t}}}.
\]
Then,
\[\label{eqn:cts-ftrl-equality}
  &\hspace{-1em}
  \ftrlobj{t}(\genrnvec{t}) - \ftrlobj{t+1}(\genrnvec{t+1}) +\inner[0]{\lossvec{t}}{\genrnvec{t}}\\
    &= \ftrlobj{t}(\genrnvec{t}) - \ftrlobj{t}(\genrnvec{t+1}) + 
    \inner[0]{\lossvec{t}}{\genrnvec{t} - \genrnvec{t+1}}
    - \Big(\frac{1}{\lr{t+1}} - \frac{1}{\lr{t}}\Big) \gendiv{\gendivfun}{\basemeasvec}{\genrnvec{t+1}}.
\]

By \cref{lem:gen-first-order-optimality} and the definition of $\genrnvec{t}$,
\*[
	\bregman{\ftrlobj{t}}{\genrnvec{t+1}}{\genrnvec{t}}
	\leq \ftrlobj{t+1}(\genrnvec{t+1}) - \ftrlobj{t}(\genrnvec{t}).
\]
Thus, substituting this into \cref{eqn:cts-ftrl-equality},
\*[
  &\hspace{-1em}
  \ftrlobj{t}(\genrnvec{t}) - \ftrlobj{t+1}(\genrnvec{t+1}) +\inner[0]{\lossvec{t}}{\genrnvec{t}}\\
    &\leq -\frac{1}{\lr{t}}\bregman{\gendiv{\gendivfun}{\basemeasvec}{\cdot}}{\genrnvec{t+1}}{\genrnvec{t}} + 
    \inner[0]{\lossvec{t}}{\genrnvec{t} - \genrnvec{t+1}}
    - \Big(\frac{1}{\lr{t+1}} - \frac{1}{\lr{t}}\Big) \gendiv{\gendivfun}{\basemeasvec}{\genrnvec{t+1}},
\]
where we have used linearity of the functional Bregman divergence.

Next, for any $\genrnvec{}, \genrndumvec{} \in \probrvspace(\basemeasvec)$,
\[\label{eqn:gen-fdiv-lim}
	\funcdiff{\gendiv{\gendivfun}{\basemeasvec}{\cdot}}{\genrndumvec{}}{\genrnvec{}-\genrndumvec{}}
	&= \lim_{\alpha \to 0}
	\int\frac{ \sbra[1]{\gendivfun\rbra[1]{\genrndum{}{\genexpidx} + \alpha(\genrn{}{\genexpidx}-\genrndum{}{\genexpidx})} - \gendivfun\rbra[1]{\genrndum{}{\genexpidx}}}}{\alpha} \basemeas{\dee\genexpidx}.
\]
For any $\genexpidx\in\genexpset$ and $\alpha\in(0,1)$, since $\gendivfun'$ is increasing, by the mean value theorem
\[\label{eqn:gen-fdiv-bounds}
	&\hspace{-2em}\gendivfun'
	\rbra[2]{\min\cbra[2]{\inf_{\genexpidxdum\in\genexpset} \genrndum{}{\genexpidxdum} + \alpha(\genrn{}{\genexpidxdum}-\genrndum{}{\genexpidxdum}), \inf_{\genexpidxdum\in\genexpset} \genrndum{}{\genexpidxdum}}} \\
	&\leq
	\frac{ \sbra[1]{\gendivfun\rbra[1]{\genrndum{}{\genexpidx} + \alpha(\genrn{}{\genexpidx}-\genrndum{}{\genexpidx})} - \gendivfun\rbra[1]{\genrndum{}{\genexpidx}}}}{\alpha} \\
	&\leq 
	\gendivfun'
	\rbra[2]{\max\cbra[2]{\sup_{\genexpidxdum\in\genexpset} \genrndum{}{\genexpidxdum} + \alpha(\genrn{}{\genexpidxdum}-\genrndum{}{\genexpidxdum}), \sup_{\genexpidxdum\in\genexpset} \genrndum{}{\genexpidxdum}}}. 
\]

Recall the form of $\genrnvec{t}$ given by \cref{lem:gen-ftrl-soln-formula}.
For each round $t$, $\sup_{\genexpidx\in\genexpset} \cbra{-\lr{t}\Loss{t-1}{\genexpidx} + \konst^*}<\infty$ so $\sup_{\genexpidx\in\genexpset}\max\cbra{\genrn{t}{\genexpidx}, \genrn{t+1}{\genexpidx}} < \infty$. 
Thus, since by continuity $\gendivfun'$ is finite on $\domfdivfun(\basemeasvec{})$, for $\genrnvec{} = \genrnvec{t+1}$ and $\genrndumvec{}=\genrnvec{t}$ the upper bound of \cref{eqn:gen-fdiv-bounds} is finite.
If $\gendivfun'(0+) > -\infty$, the lower bound is also finite.
If $\gendivfun'(0+) = -\infty$, 
then
since $\Lossvec{t-1}$ is bounded, $\genrnvec{t}$ is bounded away from $0$ uniformly and thus the lower bound of \cref{eqn:gen-fdiv-bounds} is still finite.
Thus, for any $\gendivfun$ satisfying the conditions of the lemma, we can apply the bounded convergence theorem to \cref{eqn:gen-fdiv-lim} to obtain
\*[
	\bregman{\gendiv{\gendivfun}{\basemeasvec}{\cdot}}{\genrnvec{t+1}}{\genrnvec{t}}
	= \int \bregman{\gendivfun}{\genrn{t+1}{\genexpidx}}{\genrn{t}{\genexpidx}} \basemeas{\dee\genexpidx}.
\]
Thus,
\*[
  &\hspace{-2em}
  \ftrlobj{t}(\genrnvec{t}) - \ftrlobj{t+1}(\genrnvec{t+1}) +\innermeas[0]{\lossvec{t}}{\genrnvec{t}}{\basemeasvec}\\
    &\leq \int \rbra{\loss{t}{\genexpidx}-\mean{t}} \rbra{\genrn{t}{\genexpidx} - \genrn{t+1}{\genexpidx}} 
    - \frac{1}{\lr{t}} \bregman{\gendivfun}{\genrn{t+1}{\genexpidx}}{\genrn{t}{\genexpidx}}  \\
    &\qquad - \rbra{\frac{1}{\lr{t+1}} - \frac{1}{\lr{t}}} \gendivfun(\genrn{t+1}{\genexpidx}) \basemeas{\dee\genexpidx},
\]
where we have used that any $\mean{t}$ can be added since $\genrnvec{t}\in\probrvspace(\basemeasvec)$ for all $t$.

We split the proof now to consider the two choices of intermediate points.

\emph{Intermediate Point A}\\
By the mean value remainder form of Taylor's theorem there exists $\intgenrn{t+1}{\genexpidx}\in\convhull(\genrn{t}{\genexpidx},\genrn{t+1}{\genexpidx})$ such that
\*[
	\bregman{\gendivfun}{\genrn{t+1}{\genexpidx}}{\genrn{t}{\genexpidx}}
  & = \frac{1}{2}\gendivfun''(\intgenrn{t+1}{\genexpidx}) \rbra{\genrn{t}{\genexpidx} - \genrn{t+1}{\genexpidx}}^2.
\]
That is,
\[\label{eqn:gen-ftrl-pre-fenchel-A}
  &\hspace{-2em}
  \ftrlobj{t}(\genrnvec{t}) - \ftrlobj{t+1}(\genrnvec{t+1}) +\innermeas[0]{\lossvec{t}}{\genrnvec{t}}{\basemeasvec}\\
    &\leq \int \rbra{\loss{t}{\genexpidx}-\mean{t}} \rbra{\genrn{t}{\genexpidx} - \genrn{t+1}{\genexpidx}} 
    - \frac{1}{2\lr{t}}\gendivfun''(\intgenrn{t+1}{\genexpidx}) \rbra{\genrn{t}{\genexpidx} - \genrn{t+1}{\genexpidx}}^2  \\
    &\qquad - \rbra{\frac{1}{\lr{t+1}} - \frac{1}{\lr{t}}} \gendivfun(\genrn{t+1}{\genexpidx}) \basemeas{\dee\genexpidx}.
\]

Since $\gendivfun''$ is bounded away from zero, we can apply Fenchel-Young point-wise for each $\genexpidx\in\genexpset$ to obtain
\*[
	\rbra{\loss{t}{\genexpidx}-\mean{t}}\rbra{\genrn{t}{\genexpidx} - \genrn{t+1}{\genexpidx}} 
	\leq \frac{\lr{t}}{2} \frac{\rbra{\loss{t}{\genexpidx} - \mean{t}}^2}{\gendivfun''(\intgenrn{t+1}{\genexpidx})} + \frac{1}{2\lr{t}}\gendivfun''(\intgenrn{t+1}{\genexpidx}) \rbra{\genrn{t}{\genexpidx} - \genrn{t+1}{\genexpidx}}^2.
\]
Substituting into \cref{eqn:gen-ftrl-pre-fenchel-A} and summing over $t$ gives the result.

\emph{Intermediate Point B}\\
Alternatively, let
\*[
	\orthgenrnvec{t+1} \in \argmax_{\genrnvec{} \in \posrvspace(\basemeasvec)} \cbra[2]{-\inner[0]{\lossvec{t} - \mean{t}\mathbbm{1}}{\genmeasvec{\basemeasvec}{}}
    - \frac{1}{\lr{t}} \int \bregman{\gendivfun}{\genrn{}{\genexpidx}}{\genrn{t}{\genexpidx}} \basemeas{\dee\genexpidx}},
\]
so
\*[
  &\hspace{-2em}
  \ftrlobj{t}(\genrnvec{t}) - \ftrlobj{t+1}(\genrnvec{t+1}) +\innermeas[0]{\lossvec{t}}{\genrnvec{t}}{\basemeasvec}\\
    &\leq \int \rbra{\loss{t}{\genexpidx}-\mean{t}} \rbra{\genrn{t}{\genexpidx} - \orthgenrn{t+1}{\genexpidx}}  \\
    &\qquad - \frac{1}{\lr{t}} \bregman{\gendivfun}{\orthgenrn{t+1}{\genexpidx}}{\genrn{t}{\genexpidx}} - \rbra{\frac{1}{\lr{t+1}} - \frac{1}{\lr{t}}} \gendivfun(\genrn{t+1}{\genexpidx}) \basemeas{\dee\genexpidx}.
\]

Again by the mean value remainder form of Taylor's theorem,
there exists $\inttgenrn{t+1}{\genexpidx}\in\convhull(\genrn{t}{\genexpidx},\orthgenrn{t+1}{\genexpidx})$ such that
\*[
	\bregman{\gendivfun}{\orthgenrn{t+1}{\genexpidx}}{\genrn{t}{\genexpidx}}
  & = \frac{1}{2}\gendivfun''(\inttgenrn{t+1}{\genexpidx}) \rbra{\genrn{t}{\genexpidx} - \orthgenrn{t+1}{\genexpidx}}^2.
\]
The statement follows from the same point-wise application of Fenchel-Young as for intermediate point A. 

Finally, note that for each $t$ and $\genexpidx\in\genexpset$, solving the gradient equation and applying a very similar argument to the proof of \cref{lem:gen-ftrl-soln-formula} gives
\*[
	\orthgenrn{t+1}{\genexpidx}
	= [\gendivfun']^{-1}(\gentruncfun\rbra{-\lr{t}(\loss{t}{\genexpidx}-\mean{t}) + \gendivfun'(\genrn{t}{\genexpidx})})
	\leq \genrn{t}{\genexpidx},
\]
as long as $\inf_{\genexpidx\in\genexpset} \loss{t}{\genexpidx} \geq \mean{t}$.
\end{proof}

\section{Proofs for $\fdivfun$-divergence \FTRL{} bounds}
\label{sec:proof-continuous-KL}

\ContinuousKL*

\begin{proof}[Proof of \cref{thm:continuous-KL}]
We start by applying \cref{lem:gen-ftrl-decomp} using the intermediate point $\inttgenrnvec{t+1}$ and with $\mean{t} \equiv 0$.
Since $\fdivfun'' \cdot (\fdivfun+\fconstA) \geq \fconstB$ and either $1/\fdivfun''$ or $\fdivfun$ is increasing, this simplifies to
\*[
  \regret{T}{\conttargetvec}
    &\leq \frac{1}{\lr{T+1}} \gendiv{\gendivfun}{\basemeasvec}{\rnderivflat{\gentargetvec{}}{\basemeasvec}}
    + \frac{1}{2\fconstB} \sum_{t=1}^T \lr{t} \int \fdivfun(\contrn{t}{\contexpidx})\rbra{\loss{t}{\contexpidx}}^2  \basemeas{\dee\contexpidx} 
    + \frac{\fconstA}{2\fconstB} \sum_{t=1}^T \lr{t} \int \rbra{\loss{t}{\contexpidx}}^2  \basemeas{\dee\contexpidx} \\
    &\qquad - \sum_{t=1}^T\rbra{\frac{1}{\lr{t+1}} - \frac{1}{\lr{t}}} \gendiv{\fdivfun}{\basemeasvec}{\contrnvec{t+1}}.
\]

Taking $\lr{t}=c/\sqrt{t}$ for some $c>0$ gives
\*[
	\frac{1}{\lr{t+1}} - \frac{1}{\lr{t}}
	\geq \frac{1}{2c\sqrt{t+1}}.
\]
Finally,
\*[
	\sum_{t=1}^T \frac{1}{\sqrt{t}}
	\leq 1 + \int_1^T \frac{1}{\sqrt{x}}\dee x
	\leq 2\sqrt{T}.
\]
Thus, using that $\lossvec{t} \leq 1$ and $\gendiv{\fdivfun}{\basemeasvec}{\contrnvec{t}} \geq 0$ (by Jensen's and $\fdivfun(1)=0$),
\*[
  \regret{T}{\conttargetvec}
    &\leq \frac{\sqrt{T+1}}{c} \gendiv{\gendivfun}{\basemeasvec}{\rnderivflat{\gentargetvec{}}{\basemeasvec}}
    + \sum_{t=1}^T \frac{c}{2\fconstB\sqrt{t}} \gendiv{\fdivfun}{\basemeasvec}{\contrnvec{t}}
    + \frac{\fconstA \const}{\fconstB}\sqrt{T}
    - \sum_{t=1}^T \frac{1}{2c\sqrt{t+1}} \gendiv{\fdivfun}{\basemeasvec}{\contrnvec{t+1}} \\
    &\leq \frac{\sqrt{T+1}}{c} \gendiv{\gendivfun}{\basemeasvec}{\rnderivflat{\gentargetvec{}}{\basemeasvec}}
    + \frac{\fconstA \const}{\fconstB}\sqrt{T}
    + \sum_{t=2}^T \rbra{\frac{c}{2\fconstB\sqrt{t}} \gendiv{\fdivfun}{\basemeasvec}{\contrnvec{t}}
    - \frac{1}{2c\sqrt{t}} \gendiv{\fdivfun}{\basemeasvec}{\contrnvec{t}}},
\]
where the last step uses that $\contrnvec{1}=\basemeasvec$.
Taking $c=\sqrt{\fconstB}$ makes the final summation over $t$ non-positive while minimizing the first term.
\end{proof}

We also have the following result, which implies \cref{cor:kl-bound}.

\begin{lemma}\label{lem:KL-leq-sqrtDf}
  If $\fdivfun:\PosReals\to \Reals$ and there exists $\konst_1\in\Reals$ and $\konst_2\in\PosReals$ with $f(r)\leq \konst_1+\konst_2 r \sqrt{\log(1+r)}$ for all $r\in\PosReals$, 
  then for all $\mu\in\probmeasll{\basemeasvec}\genexpspace$
  \*[
    \fdivergence{\fdivfun}{\mu}{\basemeasvec} \leq \konst_1+\konst_2\sqrt{1+\KL{\mu}{\basemeasvec}}.
  \]
\end{lemma}

\begin{proof}[Proof of \cref{lem:KL-leq-sqrtDf}]
  \*[
    \fdivergence{\fdivfun}{\mu}{\basemeasvec}
    & = \int \fdivfun\rbra{\rnderiv{\mu}{\basemeasvec}(\genexpidx)}\basemeasvec(\dee \genexpidx) \\
    & \leq \konst_1+\konst_2 \int \rnderiv{\mu}{\basemeasvec}(\genexpidx) \sqrt{1+\log\rbra{\rnderiv{\mu}{\basemeasvec}(\genexpidx)}}\, \basemeasvec(\dee \genexpidx) \\
    & = \konst_1+\konst_2 \int \sqrt{\log\rbra{1+\rnderiv{\mu}{\basemeasvec}(\genexpidx)}}\, \mu(\dee \genexpidx)\\
    &\leq \konst_1+\konst_2  \sqrt{\int\log\rbra{1+\rnderiv{\mu}{\basemeasvec}(\genexpidx)}\, \mu(\dee \genexpidx)}\\
    & = \konst_1+\konst_2  \sqrt{\int\rnderiv{\mu}{\basemeasvec}(\genexpidx)\log\rbra{1+\rnderiv{\mu}{\basemeasvec}(\genexpidx)}\, \basemeasvec(\dee \genexpidx)}\\
    &\leq \konst_1+\konst_2  \sqrt{1+\KL{\mu}{\basemeasvec}}.
  \]
  The first inequality is by assumption, the second is Jensen's, and the third is because $1+ (r\log r) \geq r\log(r+1)$ for all $r>0$.
\end{proof}

Finally, we restate and prove \cref{fact:gen-variance-bound}.

\VarianceBoundCoroll*

\begin{proof}[Proof of \cref{fact:gen-variance-bound}]
Note that, since $\gendivfun(1/\basemeas{\genexpset}) \geq 0$, Jensen's inequality implies that for any $\genrnvec{} \in \probrvspace$
\*[
  \gendiv{\gendivfun}{\basemeasvec}{\genrnvec{}}
  = \int \gendivfun(\genrn{}{\genexpidx}) \basemeas{\dee \genexpidx}
  \geq \basemeas{\genexpset} \gendivfun\rbra{\frac{1}{\basemeas{\genexpset}}\int \genrn{}{\genexpidx} \basemeas{\dee \genexpidx}}
  \geq 0.
\]

First, suppose $1/\gendivfun''(x) \leq \Const x$.
Since $\lr{t}$ is decreasing, \cref{lem:gen-ftrl-decomp} with intermediate point A implies
\*[
  \regret{T}{\gentargetvec{}}
  &\leq \frac{1}{\lr{T+1}}\gendivrn{\gendivfun}{\basemeasvec}{\rnderiv{\gentargetvec{}}{\basemeasvec}}
  + \sum_{t=1}^T \frac{\Const \lr{t}}{2} \int \intgenrn{t+1}{\genexpidx} \rbra{\loss{t}{\genexpidx} - \mean{t}}^2 \basemeas{\dee\genexpidx}.
\]
Taking $\mean{t} = \int \loss{t}{\genexpidx} \intgenrn{t+1}{\genexpidx} \basemeas{\dee\genexpidx}$ gives
\*[
  \regret{T}{\gentargetvec{}}
  &\leq \frac{1}{\lr{T+1}}\gendivrn{\gendivfun}{\basemeasvec}{\rnderiv{\gentargetvec{}}{\basemeasvec}}
  + \frac{\sqrt{\Const}}{2} \sum_{t=1}^T \frac{\Var_{\genexpidx \sim \intgenmeasvec{\basemeasvec}{t+1}} \loss{t}{\genexpidx}}{\sqrt{1/2 + \sum_{s=1}^{t-2} \Var_{\genexpidx \sim \intgenmeasvec{\basemeasvec}{s+1}} \loss{s}{\genexpidx}}} \\
  &\leq \frac{1}{\lr{T+1}}\gendivrn{\gendivfun}{\basemeasvec}{\rnderiv{\gentargetvec{}}{\basemeasvec}}
  + \frac{\sqrt{\Const}}{2} \sum_{t=1}^{T} \frac{\Var_{\genexpidx \sim \intgenmeasvec{\basemeasvec}{t+1}} \loss{t}{\genexpidx}}{\sqrt{\sum_{s=1}^{t} \Var_{\genexpidx \sim \intgenmeasvec{\basemeasvec}{s+1}} \loss{s}{\genexpidx}}},
\]
where we have used that the variance of a random variable in $[0,1]$ is bounded by $1/4$ (see, e.g., Lemma~8 of \citep{semiadv}).
By Lemma~4.13 of \citet{orabona2019}, this gives
\*[
  \regret{T}{\gentargetvec{}}
  &\leq \gendivrn{\gendivfun}{\basemeasvec}{\rnderiv{\gentargetvec{}}{\basemeasvec}}  \sqrt{\Const\sbra{1/2 + \sum_{t=1}^{T-1} \Var_{\genexpidx \sim \intgenmeasvec{\basemeasvec}{t+1}} \loss{t}{\genexpidx}}} \
  + \sqrt{\Const \sum_{t=1}^{T} \Var_{\genexpidx \sim \intgenmeasvec{\basemeasvec}{t+1}} \loss{t}{\genexpidx}}.
\]

Next, suppose $1/\gendivfun''(\genexpidx) \leq \Const$. 
Since $\lr{t}$ is decreasing, \cref{lem:gen-ftrl-decomp} with either intermediate point implies
\*[
  \regret{T}{\gentargetvec{}}
  &\leq \frac{1}{\lr{T+1}}\gendivrn{\gendivfun}{\basemeasvec}{\rnderiv{\gentargetvec{}}{\basemeasvec}}
  + \sum_{t=1}^T \frac{\basemeas{\genexpset} \Const \lr{t}}{2} \int  \rbra{\loss{t}{\genexpidx} - \mean{t}}^2 \baseprob{\dee\genexpidx}.
\]
Taking $\mean{t} = \int \loss{t}{\genexpidx} \baseprob{\genexpidx}$ gives
\*[
  \regret{T}{\gentargetvec{}}
  &\leq \frac{1}{\lr{T+1}}\gendivrn{\gendivfun}{\basemeasvec}{\rnderiv{\gentargetvec{}}{\basemeasvec}}
  + \frac{\sqrt{\basemeas{\genexpset} \Const}}{2} \sum_{t=1}^T \frac{\Var_{\genexpidx \sim \baseprobvec} \loss{t}{\genexpidx}}{\sqrt{1/4 + \sum_{s=1}^{t-1} \Var_{\genexpidx \sim \baseprobvec} \loss{s}{\genexpidx}}} \\
  &\leq \frac{1}{\lr{T+1}}\gendivrn{\gendivfun}{\basemeasvec}{\rnderiv{\gentargetvec{}}{\basemeasvec}}
  + \frac{\sqrt{\basemeas{\genexpset} \Const}}{2} \sum_{t=1}^T \frac{\Var_{\genexpidx \sim \baseprobvec} \loss{t}{\genexpidx}}{\sqrt{\sum_{s=1}^{t} \Var_{\genexpidx \sim \baseprobvec} \loss{s}{\genexpidx}}} 
\]
By Lemma~4.13 of \citet{orabona2019}, this gives
\*[
  \regret{T}{\gentargetvec{}}
  &\leq \gendivrn{\gendivfun}{\basemeasvec}{\rnderiv{\gentargetvec{}}{\basemeasvec}}  \sqrt{\basemeas{\genexpset} \Const \sbra{1/4 + \sum_{t=1}^{T} \Var_{\genexpidx \sim \baseprobvec} \loss{t}{\genexpidx}}} \
  + \sqrt{\basemeas{\genexpset} \Const \sum_{t=1}^{T} \Var_{\genexpidx \sim \baseprobvec} \loss{t}{\genexpidx}}.
\]
\end{proof}

\section{Proof of lower bound for quantile regret}
\label{sec:lower-bound}

\LowerBound*

\begin{proof}[Proof of \cref{thm:lower-bound}]
Let 
$\normalcdf(z) \defas \Pr(Z\leq z)$ for $Z\distas\normaldist(0,1)$ be the normal cumulative distribution function and
$\normalccdf(z) \defas 1-\normalcdf(z)$ be the normal complementary cumulative distribution function.
Let $\loss{t}{\expidx} \sim \bernoullidist(1/2)$ for all $t\in[T]$ and $\expidx\in\experts$.
For a sequence of real values, $S=\rbra{s_j}_{j\in\range{m}}$, denote its empirical cumulative distribution function by $\empiricalcdf{S}(x) \defas \sum_{j\in \range{m}} \ind{x\leq s_j}$.
For a non-decreasing \emph{c\`{a}dl\`{a}g} function $F$ (right continuous and left limits exist), let its improper inverse be $F^+(y) = \inf\set{x \setdelim F(x)\geq y}$.
Let $\tildeLossvec{t} = \Lossvec{t} - (t/2)\mathbbm{1}_{\numexperts}$ be the centred cumulative losses. Clearly,
\*[
  \EE_{\lossvec{1:T}\sim\lossdistn{}^{\otimes T}}\regret{T}{\orderidx{\epsidx}}
  & = - \EE\empiricalcdf{\tildeLossvec{T}}^+(\epsilon).
\]

Without loss of generality, we can enrich our probability space so that there exists \iid{} standard normal random variables $Z_T(\expidx)\distiidas\normaldist(0,1)$ with $\Loss{T}{\expidx} = F_{\binomialdist(T,1/2)}^+ \circ \normalcdf(Z_T(\expidx))$.
From \citet[Lemma~4]{bretagnolle1989hungarian} this coupling satisfies
\*[
  \abs{\Loss{T}{\expidx} - T/2 - \frac{\sqrt{T}}{2}Z_T(\expidx)} \leq 1 + \frac{Z_T(\expidx)^2}{8},
\]
which implies
\*[
  \EE \max_{\expidx\in\experts}\abs{\Loss{T}{\expidx}- T/2 - \frac{\sqrt{T}}{2}Z_T(\expidx)} \leq 1 + \frac{1}{8}
    \EE \max_{\expidx\in\experts}Z_T(\expidx)^2.
\]
Now, 
\*[
  \EE \max_{\expidx\in\experts}Z_T(\expidx)^2
    & \leq \inf_{\lambda\in\ointer{0,1/2}}\frac{1}{\lambda}\log \sum_{\expidx\in\experts}\EE \exp(\lambda Z_T(\expidx)^2)\\
    & = \inf_{\lambda\in \ointer{0,1/2}}\frac{1}{\lambda}\log \frac{\numexperts}{\sqrt{1-2\lambda}}\\
    & = \inf_{\lambda\in \ointer{0,1/2}}\rbra{\frac{\log\numexperts}{\lambda} + \frac{1}{\lambda}\log\frac{1}{\sqrt{1-2\lambda}}}.
\]
Overapproximating with $\lambda=1/4$ gives
\*[
  \EE \max_{\expidx\in\experts}Z_T(\expidx)^2
    & \leq 4\log\numexperts + 2\log 2.
\]
  
From \citet{ali1965some}, for $\epsilon > 3/4$
\*[
  - \EE\empiricalcdf{Z_T}^+(\epsilon)
    & = \EE\empiricalcdf{-Z_T}^+(1-\epsilon) \\
    & \geq \EE\normalcdf^+\rbra{\frac{\floor{(1-\epsilon) N}}{N+1}} \\
    & = \EE\normalccdf^{-1}\rbra{1-\frac{\floor{(1-\epsilon) N}}{N+1}}.
\]
Then, since the coupling is monotone,
transformations of the order statistics are equal to the order statistics of the transformation, and hence
\*[
  &\hspace{-2em}\abs{\empiricalcdf{-Z_T}^+(1-\epsilon) -  \empiricalcdf{\rbra{-\frac{2}{\sqrt{T}}\rbra{\Lossvec{T} - T/2}}}^+(1-\epsilon)} \\
    & \leq \max_{\expidx\in\experts}\abs{Z_T(\expidx) - \rbra{\frac{2}{\sqrt{T}}\rbra{\Loss{T}{\expidx} - T/2}}}.
\]
Thus
\*[
\EE \abs{\empiricalcdf{-Z_T}^+(1-\epsilon) -  \empiricalcdf{\rbra{-\frac{2}{\sqrt{T}}\rbra{\Lossvec{T} - T/2}}}^+(1-\epsilon)}
  & \leq  \frac{4\log\numexperts + 2\log 2}{\sqrt{T}},
\]
so
\*[
  \EE\empiricalcdf{\rbra{-\frac{2}{\sqrt{T}}\rbra{\Lossvec{T} - T/2}}}^+(1-\epsilon) \geq \normalccdf^{-1}\rbra{1-\frac{(1-\epsilon) \numexperts}{\numexperts+1}} - \frac{4\log\numexperts + 2\log 2}{\sqrt{T}},
\]
which implies
\*[
  -\EE\empiricalcdf{\tildeLossvec{T}}^+(\epsilon) &\geq \frac{\sqrt{T}}{2}\normalccdf^{-1}\rbra{\frac{\numexperts}{\numexperts+1}\epsilon + \frac{1}{\numexperts+1}} - 2\log\numexperts - \log 2\\
  &\geq \frac{\sqrt{T}}{2}\normalccdf^{-1}\rbra{2\epsilon} - 2\log\numexperts - \log 2.
\]

Now, applying \cref{lem:tight-normal-tail-lowerbound}, for $\epsilon<1/4$
\*[
  -\EE\empiricalcdf{\tildeLossvec{T}}^+(\epsilon) \geq \sqrt{\frac{T}{2}\rbra{  \log(1/\epsilon) -2\log2+1/\pi}} - \sqrt{2/\pi} - 2\log\numexperts - \log 2.
\]
\end{proof}

\begin{lemma}\label{lem:tight-normal-tail-lowerbound}
For $x>0$,
\*[
  \normalccdf(x) \geq \frac{\exp(1/\pi)\exp\rbra[0]{-(x+\sqrt{2/\pi})^2/2}}{2},
\]
and for $0<y<1/2$,
\*[
  \normalccdf^{-1}(y) \geq \sqrt{2 \log(1/y) -2\log2+2/\pi} - \sqrt{2/\pi}.
\]  
\end{lemma}

\begin{proof}
The second bound follows from the first, so we will only prove the first.
Let $\normalpdf(x)$ denote $\normalcdf'(x) = \exp(-x^2/2) / \sqrt{2\pi}$, and
\*[
  h(x) 
    & = \log\normalccdf(x) - \log\frac{\exp(1/\pi)\exp\rbra[0]{-(x+\sqrt{2/\pi})^2/2}}{2} \\
    &= \log\normalccdf(x)- {1}/{\pi}+\log(2) +(x+\sqrt{2/\pi})^2/2.
\]
Then, $h(0)=0$ and
\*[
  h'(x) = -\normalpdf(x)/\normalccdf(x) + (x+\sqrt{2/\pi}),
\]
so $h'(0)=0$ as well.

Next,
\*[
  h''(x) = x\normalpdf(x)/\normalccdf(x) - \rbra{\normalpdf(x)/\normalccdf(x)}^2+1,
\]
so $h''(0)=1-2/\pi>0$.

Suppose towards a contradiction that $h''(x_0)=0$ for some $x_0>0$. Rearranging gives
\*[
  x_0 = \normalpdf(x_0)/\normalccdf(x_0) -\rbra{\normalpdf(x_0)/\normalccdf(x_0)}^{-1},
\]
so
\*[
  \normalpdf(x_0)/\normalccdf(x_0) = \frac{\sqrt{4+x_0^2}+x_0}{2}.
\]
However, from \citet{birnbaum1942inequality}\footnote{The stated result has a non-strict inequality, however the proof by Jensen's inequality clearly leads to a strict inequality since the function in the application of Jensen's inequality, $t\mapsto 1/t$ is strictly convex and the distribution to which it is applied is not degenerate, it has density proportional to $\phi'(t)$ for $t>0$.}, for all $x>0$
\*[
  \normalpdf(x)/\normalccdf(x) < \frac{\sqrt{4+x^2}+x}{2}.
\] 
Thus $h''(x)\neq 0$ for all $x>0$. Since $h''$ is continuous, $h''(0)>0$, and $h''$ has no zeros on $\ointer{0,\infty}$, by the intermediate value theorem $h''(x)>0$ for all $x>0$.
Finally, since $h(0)=h'(0)=0$ and $h''(x)>0$ for all $x>0$, by standard differential inequalities, $h(x)\geq0$ for all $x\geq0$.
\end{proof}

\section{Proof of semi-adversarial regret bound}
\label{sec:proof-discrete-semidav}

To analyze \ABN{}, we define two more notions to characterize which stage of learning the algorithm is in. Note that each of these stages occur \emph{implicitly} while running \ABN{}, and consequently the user does not need to know any of the quantities associated with $\distnball{}$. 

For each $\expidx\in\neffexperts$, let $\timeexp{\expidx} = \ceil{8(\log \numexperts) / \Deltaexp{\expidx}^2}$
and $\timeeffexp = \max_{\expidx\in\neffexperts}\timeexp{\expidx} $.
For each $t\in\Nats$ let $\effexpertstime{t} = \effexperts\union \set{\expidx\in\neffexperts \stT \timeexp{\expidx}\geq t}$
and $\numeffexpertstime{t} = \card[1]{\effexpertstime{t}}$.
For each $\expsumidx \in \set{0,1,...,\numexperts-\numeffexperts-1}$, let $\timeordexp{\expsumidx}$ be the decreasing ordered values of $T_{i}$, 
so that $\timeeffexp = \timeordexp{0} \geq \timeordexp{1} \geq ... \geq \timeordexp{\numexperts-\numeffexperts-1}$.
Let $\timeordexp{\numexperts-\numeffexperts} = 0$.
Let $\Deltaordexp{\expsumidx}$ be the corresponding increasing ordered values of $\Deltaexp{\expidx}$,
so that $\Deltaeff = \Deltaordexp{0} \leq \Deltaordexp{1} \leq ... \leq \Deltaordexp{\numexperts-\numeffexperts-1}$.

We then have the following regret bound for \ABN{}, which simplifies to \cref{fact:simple-carl-bound} by lower bounding all $\Delta_\expidx$ by $\Deltaeff$ and observing that $W_{j,\numexperts,\numeffexperts}$ telescopes in the sum. The more refined bound of \cref{thm:discrete-semiadv} will be significantly tighter for small $T$ when only one $\Delta_\expidx$ is small, although the worst-case adversary will choose all $\Delta_\expidx = \Deltaeff$.

\begin{theorem}
\label{thm:discrete-semiadv}
For any time-homogeneous convex constraint $\distnball{}$, \ABN{} achieves:\\
For all $T$,
\*[
\EE \bestregret{T}
    &\leq \sqrt{2T\log\numexperts},
\]
and if $T>\timeeffexp$,
\*[
\EE \bestregret{T}
    &\leq\sqrt{2T\log \numeffexperts}
    + 4(\log\numexperts) \sum_{j=0}^{\numexperts-\numeffexperts-1} W_{j,\numexperts,\numeffexperts} \frac{1}{\Deltaordexp{\expsumidx}} \\
    &\qquad + \frac{5\sqrt{2}}{\numexperts  \sqrt{\log\numexperts}} \rbra[1]{e^{-1/2} + \ind{\numeffexperts=1}}
    \sum_{\expidx\in\neffexperts} \frac{\ind{T>T_i}}{\Deltaexp{\expidx}}   + \sqrt{\log\numexperts},
\]
where $W_{j,\numexperts,\numeffexperts} = \tfrac{1}{\sqrt{\log\numexperts}} \rbra{\sqrt{\log(\numeffexperts+j+1)} - \sqrt{\log(\numeffexperts+j)}}$.

\end{theorem}

\begin{proof}[Proof of \cref{thm:discrete-semiadv}]

Since $\numexperts < \infty$, we wish to apply \cref{lem:gen-ftrl-soln-formula,lem:gen-ftrl-decomp} for $\basemeasvec$ equal to counting measure and $\domfdivfun(\basemeasvec) = [0,1]$.
It is straightforward to verify that $-\partialentropyB''(\rndumidx) = 1/\partialentropyA(\rndumidx) > 0$ on $(0,1)$.
Thus, we can apply \cref{lem:gen-ftrl-decomp} with intermediate point A, so there exists a sequence $\intweightvec{t+1}\in \convhull(\ABNweightvec{t},\ABNweightvec{t+1})$ such that for any sequence $(\mean{t})_{t\in\Nats}\subseteq\Reals$, it holds almost surely that
\*[
  \bestregret{T}
  &\leq \sqrt{\log\numexperts}
  + \sum_{t=1}^T \sum_{\expidx\in\experts} \sbra{\frac{\const}{2\sqrt{t}} \partialentropyA(\intweight{t+1}{\expidx}) \rbra{\loss{t}{\expidx} - \mean{t}}^2 + \frac{1}{\const}\rbra[0]{\sqrt{t+1}-\sqrt{t}} \entropyB(\ABNweight{t+1}{\expidx})},
\]
where we have used that $\entropyB(\weightvec{}) = 0$ for any one-hot $\weightvec{}$, an application of \cref{lem:entropyAB-bound}, and the expression for $\abdivfun''$ again. 

Let $\partialentropyA(x) = x\sqrt{2 \log(1/x)}$ for $x\in(0,1]$ and $\partialentropyA(0)$ = 0, and set $\entropyA(\weightvec{}) = \sum_{\expidx\in\experts}\partialentropyA(\weight{}{\expidx})$.
Define $\underweightmeasfun: \simp(\experts) \to \simp(\experts)$ by
\*[
  \underweightmeasfun(\weightvec{})
  &= \begin{cases}
    (1/\numexperts,\dots,1/\numexperts),
    & \entropyA(\weightvec{}) = 0 \\
    \rbra{\frac{\partialentropyA(\weight{}{\expidx})}{\entropyA(\weightvec{})}}_{\expidx\in\experts},
    & \text{otherwise}.
    \end{cases}
\]
Letting $\mean{t} = \inner{\lossvec{t}}{\underweightmeasfun(\intweightvec{t+1})}$ then gives
\*[
  \bestregret{T}
  &\leq \sqrt{\log\numexperts}
  + \sum_{t=1}^T \sbra{\frac{\const}{2\sqrt{t}} \entropyA(\intweightvec{t+1}) \VVar{\Expidx\sim \underweightmeasfun(\intweightvec{t+1})}[\loss{t}{\Expidx}] + \frac{1}{\const}\rbra[0]{\sqrt{t+1}-\sqrt{t}} \entropyB(\ABNweightvec{t+1})},
\]

Since $0 \leq \lossvec{t} \leq 1$, $\VVar{\Expidx\sim \underweightmeasfun(\weightvec{})}[\loss{t}{\Expidx}] \leq 1/4$ for any $\weightvec{}\in\simp(\experts)$, which combined with $\sqrt{t+1} - \sqrt{t} \leq 1/(2\sqrt{t})$ gives
\[\label{eqn:semiadv-pre-concentration}
  \bestregret{T}
  &\leq \sqrt{\log\numexperts}
  +  \sum_{t=1}^T \sbra{\frac{\const}{8\sqrt{t}} \entropyA(\intweightvec{t+1}) + \frac{1}{2\const\sqrt{t}} \entropyB(\ABNweightvec{t+1})}.
\]

Applying \cref{lem:entropyAB-bound} to \cref{eqn:semiadv-pre-concentration} with any $p$ and $\effexperts=\experts$ gives the worst-case bound of
\*[
\bestregret{T}
	&\leq \sqrt{\log\numexperts}
  + \rbra{\frac{\const}{4} + \frac{1}{\const}} \sum_{t=1}^T\frac{1}{2\sqrt{t}}\sqrt{2\log \numexperts}
	\leq \rbra{\frac{\const}{4} + \frac{1}{\const}} \sqrt{2 T \log \numexperts}.
\]

Next, applying \cref{lem:entropyAB-bound} to \cref{eqn:semiadv-pre-concentration} with $p=1/2$ and $\effexperts=\effexperts\upper{t}$,
\*[
\bestregret{T}
	&\leq \sqrt{\log\numexperts}
  + \rbra{\frac{\const}{4} + \frac{1}{\const}} \sum_{t=1}^T\frac{1}{2\sqrt{t}}\sqrt{2\log \numeffexperts\upper{t}} \\
		&\qquad + \sum_{t=1}^T \frac{\const}{8\sqrt{t}} \sbra{\frac{1}{\sqrt{e/2}} \sum_{\expidx\in\neffexperts\upper{t}} \sqrt{\intweight{t+1}{\expidx}}
		+ \ind{\numeffexperts\upper{t}=1} \sqrt{2}\sum_{\expidx\in\neffexperts\upper{t}} \sqrt{\intweight{t+1}{\expidx}}}\\
		&\qquad + \sum_{t=1}^T \frac{1}{2\const\sqrt{t}} \sbra{\frac{1}{\sqrt{e/2}} \sum_{\expidx\in\neffexperts\upper{t}} \sqrt{\ABNweight{t+1}{\expidx}}
		+ \ind{\numeffexperts\upper{t}=1} \sqrt{2}\sum_{\expidx\in\neffexperts\upper{t}} \sqrt{\ABNweight{t+1}{\expidx}}}.
\]
Since $\intweightvec{t+1}\in \convhull(\ABNweightvec{t},\ABNweightvec{t+1})$, then 
$\sqrt{\intweightvec{t+1}} \leq \sqrt{\ABNweightvec{t}}+ \sqrt{\ABNweightvec{t+1}}$, and so
\*[
\bestregret{T}
	&\leq \sqrt{\log\numexperts}
  + \rbra{\frac{\const}{4} + \frac{1}{\const}} \sum_{t=1}^T\frac{1}{2\sqrt{t}}\sqrt{2\log \numeffexperts\upper{t}} \\
		&\qquad + \sum_{t=1}^T \frac{\const}{8\sqrt{t}} \sbra{\frac{1}{\sqrt{e/2}} \sum_{\expidx\in\neffexperts\upper{t}} \sqrt{\ABNweight{t}{\expidx}}
		+ \ind{\numeffexperts\upper{t}=1} \sqrt{2}\sum_{\expidx\in\neffexperts\upper{t}} \sqrt{\ABNweight{t}{\expidx}}}\\
		&\qquad + \sum_{t=1}^T \rbra{\frac{\const}{8\sqrt{t}}+\frac{1}{2\const\sqrt{t}}} \sbra{\frac{1}{\sqrt{e/2}} \sum_{\expidx\in\neffexperts\upper{t}} \sqrt{\ABNweight{t+1}{\expidx}}
		+ \ind{\numeffexperts\upper{t}=1} \sqrt{2}\sum_{\expidx\in\neffexperts\upper{t}} \sqrt{\ABNweight{t+1}{\expidx}}}.
\]

Simplifying gives
\*[
\bestregret{T}
		&\leq \sqrt{\log\numexperts}
  + \rbra{\frac{\const}{4} + \frac{1}{\const}} \sum_{t=1}^T\frac{1}{2\sqrt{t}}\sqrt{2\log \numeffexperts\upper{t}} \\
			&\qquad + \frac{1}{\sqrt{e/2}} \sum_{\expidx\in\neffexperts} \sum_{t=\timeexp{\expidx}+1}^T \frac{1}{2\sqrt{t}} \sbra{ \frac{\const}{4} \sqrt{\ABNweight{t}{\expidx}}
			+ \rbra{\frac{1}{\const}+\frac{\const}{4}} \sqrt{\ABNweight{t+1}{\expidx}} 
			} \\
			&\qquad + \ind{\numeffexperts=1} \sqrt{2} \sum_{\expidx\in\neffexperts} \sum_{t=\timeeffexp+1}^T \frac{1}{2\sqrt{t}} \sbra{ \frac{\const}{4}\sqrt{\ABNweight{t}{\expidx}}
			+ \rbra{\frac{1}{\const}+\frac{\const}{4}} \sqrt{\ABNweight{t+1}{\expidx}}
			}.
\]
For $T>\timeeffexp$, the first term can be decomposed as
\*[
\sum_{t=1}^T\frac{1}{2\sqrt{t}}\sqrt{2\log \numeffexperts\upper{t}}
	& = \sum_{t=1}^{\timeeffexp}\frac{1}{2\sqrt{t}}\sqrt{2\log \numeffexperts\upper{t}}
		+ \sum_{t=\timeeffexp+1}^{T}\frac{1}{2\sqrt{t}}\sqrt{2\log \numeffexperts} \\
	& \leq
	\sum_{t=1}^{\timeeffexp}\frac{1}{2\sqrt{t}}\sqrt{2\log \numeffexperts\upper{t}}
  + \sqrt{2T\log \numeffexperts}  - \sqrt{2\timeeffexp\log \numeffexperts}.
\]

For the second and third terms, applying \cref{lem:ineffective-weight-tail}, for each $\expidx\in\neffexperts$ we have 
\*[
& \hspace{-1em}
	\EE\sum_{t=\timeexp{\expidx}+1}^T \frac{1}{2\sqrt{t}} \sbra{ \frac{\const}{4} \sqrt{\ABNweight{t}{\expidx}}
	+ \rbra{\frac{1}{\const}+\frac{\const}{4}} \sqrt{\ABNweight{t+1}{\expidx}} 
	}\\
  &\leq \frac{\const}{8} \sum_{t = \timeexp{\expidx}+1}^T \frac{1}{\sqrt{t}} \EE \sbra{\rbra{\ABNweight{t}{\expidx}}^p} + \rbra{\frac{1}{2\const} + \frac{\const}{8}} \sum_{t = \timeexp{\expidx}+2}^T \frac{2}{\sqrt{t}} \EE \sbra{\rbra{\ABNweight{t}{\expidx}}^p} \\
	&\leq 
  \rbra{\frac{1}{\const} + \frac{3\const}{8}}
  \frac{\sqrt{8}(2 + \const)}{\numexperts \Deltaexp{\expidx} \const \sqrt{\log\numexperts}}.
\]

Taking $c=2$ and combining these results, we have
\*[
\EE \bestregret{T}
    &\leq \rbra{\sqrt{2T\log \numeffexperts}  - \sqrt{2\timeeffexp\log \numeffexperts}}\ind{T>\timeeffexp} +\sum_{t=1}^{\min(T,\timeeffexp)}\frac{1}{2\sqrt{t}}\sqrt{2\log \numeffexperts\upper{t}} \\
    &\qquad +
  \frac{5\sqrt{2}}{\numexperts  \sqrt{\log\numexperts}} \rbra{e^{-1/2} + \ind{\numeffexperts=1}} \sum_{\expidx\in\neffexperts} \frac{\ind{T>\timeexp{\expidx}}}{\Deltaexp{\expidx}} + \sqrt{\log\numexperts}.    
\]

Next, using summation by parts,
\*[
	&\hspace{-1em}
	\sum_{t=1}^{\timeeffexp}\frac{1}{2\sqrt{t}}\sqrt{2\log \numeffexperts\upper{t}}\\
		& = \sum_{j=0}^{\numexperts - \numeffexperts}\sum_{t=T_{(\numexperts - \numeffexperts-j)}+1}^{T_{(\numexperts - \numeffexperts-j+1)}}\frac{1}{2\sqrt{t}}\sqrt{2\log(\numexperts-j)} \\
		& \leq 
    \sum_{j=0}^{\numexperts - \numeffexperts}\rbra{\sqrt{T_{(\numexperts - \numeffexperts-j+1)}} 
    - \sqrt{T_{(\numexperts - \numeffexperts-j)}+1} \, } \sqrt{2 \log(\numexperts - j)} \\
    &= \sqrt{2\timeordexp{1} \log\numeffexperts} - \sqrt{2\timeordexp{\numexperts-\numeffexperts} \log\numexperts} \\
    &\qquad + \sum_{j=1}^{\numexperts-\numeffexperts}
    \sqrt{T_{(\numexperts-\numeffexperts-j)}}
    \rbra{\sqrt{2\log(\numexperts-j+1)} - \sqrt{2\log(\numexperts-j)}} \\
    &\leq \sqrt{2 \timeeffexp\log\numeffexperts}
    + \sum_{j=0}^{\numexperts-\numeffexperts-1} \sqrt{\timeordexp{\expsumidx}} \rbra{\sqrt{2 \log(\numeffexperts+j+1)} - \sqrt{2 \log(\numeffexperts+j)} } \\
    &\leq 
    \sqrt{2 \timeeffexp\log\numeffexperts}
    + 4\sqrt{\log\numexperts} \sum_{j=0}^{\numexperts-\numeffexperts-1} \frac{1}{\Deltaordexp{\expsumidx}} \rbra{\sqrt{\log(\numeffexperts+j+1)} - \sqrt{\log(\numeffexperts+j)} }.
\]

In summary, for all $T>0$
\*[
\EE \bestregret{T}
		&\leq \sum_{t=1}^{T}\frac{1}{2\sqrt{t}}\sqrt{2\log \numeffexperts\upper{t}} + \sqrt{\log\numexperts} \\
    &\qquad + \frac{5\sqrt{2}}{\numexperts  \sqrt{\log\numexperts}} \rbra{e^{-1/2} + \ind{\numeffexperts=1}}
		\sum_{\expidx\in\neffexperts} \frac{\ind{T>T_i}}{\Deltaexp{\expidx}},
\]
and for $T>\timeeffexp$,
\*[
\EE \bestregret{T}
		&\leq\sqrt{2T\log \numeffexperts}
    + 4(\log\numexperts) \sum_{j=0}^{\numexperts-\numeffexperts-1} W_{j,\numexperts,\numeffexperts} \frac{1}{\Deltaordexp{\expsumidx}} \\
    &\qquad + \frac{5\sqrt{2}}{\numexperts  \sqrt{\log\numexperts}} \rbra{e^{-1/2} + \ind{\numeffexperts=1}}
    \sum_{\expidx\in\neffexperts} \frac{\ind{T>T_i}}{\Deltaexp{\expidx}} + \sqrt{\log\numexperts},
\]
where
\*[
	W_{j,\numexperts,\numeffexperts}
		& = \frac{\sqrt{\log(\numeffexperts+j+1)} - \sqrt{\log(\numeffexperts+j)}}{\sqrt{\log\numexperts}}.
\]
Note that $\sum_{j=0}^{\numexperts-\numeffexperts-1} W_{j,\numexperts,\numeffexperts} \leq 1$.
\end{proof}

\subsection{Bounding the weights of ineffective experts}

Let $\Expidxoptpath{t} = \argmin_{\expidx\in\experts}\Loss{t}{\expidx}$, so $\orderidx{1}$ is the one-hot vector with a one on $\Expidxoptpath{T}$. Then, we have the following control on the weights.

\begin{lemma}
\label{lem:local-norm-weight-eqn}
For any $\expidx\in\experts$, $\ABNweightvec{t+1}$ satisfies
\*[
  \ABNweight{t+1}{\expidx}
    & \leq
      \exp\cbra{-\const^2\frac{\rbra{\Loss{t}{\expidx} - \Loss{t}{\Expidxoptpath{t}}}^2}{2(t+1)}}.
\]
\end{lemma}

\begin{proof}[Proof of \cref{lem:local-norm-weight-eqn}]
The Lagrangian corresponding to the optimization problem defining $\ABNweightvec{t+1}$ is
\*[
  \Ll(\weightvec{}; \lambda, \alpha_1,\dots,\alpha_\numexperts)
    & = \inner{\Lossvec{t}}{\weightvec{}} - \frac{\sqrt{t+1}}{\const}\, \entropyB(\weightvec{}) + \lambda \rbra{1 - \inner{\one}{\weightvec{}}} + \sum_{\expidx\in\experts}\alpha_\expidx \weight{}{\expidx}.
\]
It is straightforward to verify that $\partialentropyB'(x) = \sqrt{2\log(1/x)}$, so
\*[
  \frac{\partial}{\partial \weight{}{\expidx}} \Ll(\weightvec{}; \lambda, \alpha_1,\dots,\alpha_\numexperts)
    & = \Loss{t}{\expidx} - \frac{\sqrt{t+1}}{\const}\, \sqrt{2\log(1/ \weight{}{\expidx})} - \lambda +\alpha_\expidx.
\]
Setting this to $0$ gives
\*[
  \lambda^* & = 
    \Loss{t}{\expidx} - \frac{\sqrt{t+1}}{\const}\, \sqrt{2\log(1/ \weight{}{\expidx})} + \alpha_\expidx.
\]

Since $\entropyB$ is invariant to reordering of the expert indices, $\ABNweightvec{t+1}$ is monotonically non-decreasing in $\Lossvec{t}$, so $\ABNweight{t+1}{\Expidxoptpath{t}} \geq 1/\numexperts$ and consequently $\alpha_{\Expidxoptpath{t}} = 0$.
If $\ABNweight{t+1}{\expidx} = 0$, the statement of the lemma trivially holds. Otherwise, $\alpha_\expidx=0$, so by definition of $\lambda^*$ we have
\*[
  \sqrt{2\log(1/\ABNweight{t+1}{\expidx})} - \sqrt{2\log(1/\ABNweight{t+1}{\Expidxoptpath{t}} )}
    & = \frac{\const}{\sqrt{t+1}}\rbra{\Loss{t}{\expidx} - \Loss{t}{\Expidxoptpath{t}}}.
\]
The result follows by rearranging this equation and using that $\sqrt{a} - \sqrt{b} \leq \sqrt{a-b}$ for all $a\geq b\geq0$ and $\ABNweight{t+1}{\expidx} \leq \ABNweight{t+1}{\Expidxoptpath{t}}$.
\end{proof}

Next, we extend Theorem~1 of \citet{semiadv} to depend on the individual ineffective experts effective gaps.

\begin{lemma}\label{fact:minimax_mgf}
For all $\numexperts\geq 2$, 
convex sets $\distnball{} \subseteq \meas(\expspace \times \dataspace)$,
$\lambda > 0$,
$t$,
and $\expidx\in\neffexperts$,
if the environment is subject to the time-homogeneous convex constraint then
\*[
  \EE \,
  \exp\cbra[2]{\lambda \rbra{\Loss{t}{\Expidxoptpath{t}} - \Loss{t}{\expidx}}}
  \leq \exp\cbra[2]{T \sbra{\lambda^2/2 - \lambda \Deltaexp{\expidx}}}.
\]
\end{lemma}
\begin{proof}[Proof of {\cref{fact:minimax_mgf}}]
The argument follows nearly identically to the one given by \citet{semiadv}. However, in the second last step in their argument, note that
\*[
  \inf_{\weightvec{}\in\simp(\effexperts)}
  \EE_{\lossvec{}} \exp\cbra[2]{\lambda[\inner{\loss{}{\effexperts}}{\weightvec{}} - \loss{}{\expidx}]}
  \leq e^{\lambda^2/2 - \lambda \Deltaexp{\expidx}}.
\]
The last step is then applied in exactly the analogous way.
\end{proof}

Using these two results, we control the expected size of the weight of an ineffective expert.

\begin{lemma}
\label{lem:ineffective-weight-bound}
For every $t$, $\expidx\in\neffexperts$, and $p>0$
\*[
  \EE \sbra{\rbra{\ABNweight{t}{\expidx}}^p}
    & \leq 2\exp\cbra{-t \frac{p \const^2}{2(\sqrt{2}+\const\sqrt{p})^2} \Deltaexp{\expidx}^2}.
\]
\end{lemma}
\begin{proof}[Proof of \cref{lem:ineffective-weight-bound}]
For any $s\in\ointer{0,1}$,
  \*[
    \EE \sbra{\rbra{\ABNweight{t}{\expidx}}^p}
      & \leq \EE \sbra{\rbra{\ABNweight{t}{\expidx}}^p \setdelim \Loss{t}{\expidx} - \Loss{t}{\Expidxoptpath{t}} \geq s \Deltaexp{\expidx} t }
        + \Pr\sbra{\Loss{t}{\expidx} - \Loss{t}{\Expidxoptpath{t}} < s \Deltaexp{\expidx} t }.
  \]
  From \cref{lem:local-norm-weight-eqn},
  \*[
  &\hspace{-1em}
    \EE \sbra{\rbra{\ABNweight{t}{\expidx}}^p \setdelim \Loss{t}{\expidx} - \Loss{t}{\Expidxoptpath{t}} \geq s \Deltaexp{\expidx} t } \\
      & \leq \EE \sbra{\exp\cbra{-p\const^2\frac{\rbra{\Loss{t}{\expidx} - \Loss{t}{\Expidxoptpath{t}}}^2}{2(t+1)}} \Biggsetdelim \Loss{t}{\expidx} - \Loss{t}{\Expidxoptpath{t}} \geq s \Deltaexp{\expidx} t } \\
      & \leq \exp\cbra{- p s^2 \const^2 \Deltaexp{\expidx}^2 \frac{t^2}{2(t+1)}} \\
      & \leq \exp\cbra{- tp s^2 \const^2 \Deltaexp{\expidx}^2 /4} .
  \]
  Next, using the Cram\`er-Chernoff Method and \cref{fact:minimax_mgf},
  \*[
  &\hspace{-1em}
  \Pr\sbra{\Loss{t}{\expidx} - \Loss{t}{\Expidxoptpath{t}} < s \Deltaexp{\expidx} t }\\
    & = \Pr\sbra{\Loss{t}{\Expidxoptpath{t}} - \Loss{t}{\expidx} > - s \Deltaexp{\expidx} t }\\
    & \leq \inf_{\lambda>0} \frac{\EE \exp\cbra{\lambda\rbra{\Loss{t}{\Expidxoptpath{t}} - \Loss{t}{\expidx}}}}{\exp\cbra{-\lambda s \Deltaexp{\expidx} t }}\\
    & \leq \inf_{\lambda>0} \frac{\exp\cbra{t\rbra{\lambda^2/2 - \lambda\Deltaexp{\expidx}}}}{\exp\cbra{-\lambda s \Deltaexp{\expidx} t }} \\
    &= \inf_{\lambda>0} \exp\cbra{t\rbra{\lambda^2/2 - (1-s)\lambda\Deltaexp{\expidx}}} \\
    &= \exp\cbra{-t\ (1-s)^2\Deltaexp{\expidx}^2 /2},
  \]
  where the last step follows from taking $\lambda = \Deltaexp{\expidx} (1-s)$.
  These two terms decay at the same rate if $p s^2 \const^2= 2(1-s)$; that is, $s = \sqrt{2} /(\sqrt{2} +\const\sqrt{p})$. Using this value,
  \*[
    \EE \sbra{\rbra{\ABNweight{t}{\expidx}}^p}
      & \leq 2\exp\cbra{-t \frac{p \const^2}{2(\sqrt{2}+\const\sqrt{p})^2} \Deltaexp{\expidx}^2}.
  \]
\end{proof}

This single-round bound can now be summed to compute the overall regret contribution of ineffective experts.
\begin{lemma}
\label{lem:ineffective-weight-tail}
For every $\expidx\in\neffexperts$ and $p>0$,
\*[
  \sum_{t = \timeexp{\expidx}+1}^\infty \frac{1}{\sqrt{t}} \EE \sbra{\rbra{\ABNweight{t}{\expidx}}^p}
  &\leq 
  \frac{2\sqrt{2}(\sqrt{2} + \const\sqrt{p})}{\numexperts \Deltaexp{\expidx} \const \sqrt{p \, \log\numexperts}}.
\]
\end{lemma}
\begin{proof}[Proof of \cref{lem:ineffective-weight-tail}]
  Using \cref{lem:ineffective-weight-bound}, we have
\*[
  \sum_{t = \timeexp{\expidx}+1}^\infty \frac{1}{\sqrt{t}} \EE \sbra{\rbra{\ABNweight{t}{\expidx}}^p}
  &\leq 2 \sum_{t = \timeexp{\expidx}+1}^\infty \frac{1}{\sqrt{t}} \exp\cbra{-t \frac{p \const^2}{2(\sqrt{2}+\const\sqrt{p})^2} \Deltaexp{\expidx}^2} \\
  &\leq 2 \int_{\timeexp{\expidx}}^\infty \frac{1}{\sqrt{t}} \exp\cbra{-t \frac{p \const^2}{2(\sqrt{2}+\const\sqrt{p})^2} \Deltaexp{\expidx}^2} \dee t.
\]
For any $\Const>0$, substituting $r^2/2 = t\Const$ and using $\Pr(Z>a)\leq \frac{\exp(-a^2/2)}{a\sqrt{2\pi}}$ for $Z\distas\normaldist(0,1)$ gives
\*[
  \int_{\timeexp{\expidx}}^\infty \frac{1}{\sqrt{t}}\exp\cbra{-t\Const} \dee t
    & = \sqrt{2/\Const} \int_{\sqrt{2 \Const \timeexp{\expidx}}}^\infty \exp\cbra{-r^2/2} \dee r \\
    & \leq \sqrt{2/\Const}\ \frac{\exp\cbra{-\Const \timeexp{\expidx}}}{\sqrt{2\Const \timeexp{\expidx}}} \\
    & = \frac{\exp\cbra{-\Const \timeexp{\expidx}}}{\Const \sqrt{\timeexp{\expidx}}}.
\]
Thus,
\*[
  \sum_{t = \timeexp{\expidx}+1}^\infty \frac{1}{\sqrt{t}} \EE \sbra{\rbra{\ABNweight{t}{\expidx}}^p}
  &\leq 2 \frac{\exp\cbra{-\timeexp{\expidx} \frac{p \const^2}{2(\sqrt{2}+\const\sqrt{p})^2} \Deltaexp{\expidx}^2}}{\frac{p \const^2}{2(\sqrt{2}+\const\sqrt{p})^2} \Deltaexp{\expidx}^2 \sqrt{\timeexp{\expidx}}} \\
  &\leq \frac{2\sqrt{2}(\sqrt{2} + \const\sqrt{p})}{\numexperts \Deltaexp{\expidx} \const \sqrt{p \, \log\numexperts}}.
\]
\end{proof}

\subsection{Entropic concentration of \ABN{} regularizer}

We now show that the functions used to define the regularizer for $\ABNweightvec{}$ concentrate like an entropy. This result is the analogue of Lemma~1 of \citet{semiadv}.

\begin{lemma}
  \label{lem:entropyAB-bound}
  For every $\weightvec{} \in \simp(\experts)$, $\effexperts\subseteq\experts$ with $\card{\effexperts}=\numeffexperts$, and $p \in (0,1)$,
  \*[\label{eqn:modular_entropyA}
  	0\leq
  	\entropyB(\weightvec{})\leq \entropyA(\weightvec{})
  	& \leq \sqrt{2\log\numeffexperts}
	+ \frac{1}{\sqrt{e(1-p)}} \sum_{\expidx\in\neffexperts} \weight{}{\expidx}^p
	+ \ind{\numeffexperts=1}\sqrt{2}\sum_{\expidx\in\neffexperts} \sqrt{\weight{}{\expidx}}.
  \]
\end{lemma}
\begin{proof}[Proof of \cref{lem:entropyAB-bound}]
Using the previously observed expressions for the first and second derivatives, $\partialentropyB$ is concave on $(0,1)$, so $\entropyB$ is concave on $\simp(\experts)$.
Bauer's maximum principle \citep{bauer60} says that a concave function on a convex, compact set $S$ attains its minimum at an extreme point of $S$.
This means that $\entropyB$ attains its minimum in the set of extreme points of $\simp([\numexperts])$, which corresponds to the set of standard basis vectors. In particular, letting $e_\expidx$ denote the $\expidx$th basis vector,
\*[
	\entropyB(\weightvec{})
	&\geq \min_{\expidx \in \experts} \entropyB(e_\expidx)
	= 0.
\]

For the second inequality, let
\*[
  g(x) = 
  \begin{cases}
  \sqrt{\frac{\pi}{2}} \erf\Big(\sqrt{\log(1/x)} \Big) - x (\numexperts-1)\sqrt{\frac{\pi}{2}},
  & x\in(0,1] \\
  \sqrt{\pi/2},
  & x=0.
  \end{cases}
\]
so that
\*[
  \entropyB(\weightvec{})
  = \entropyA(\weightvec{})
  - \sum_{\expidx\in\experts} g(\weight{}{\expidx}).
\]
Then,
\*[
	\frac{\dee^2}{\dee x^2} g(x)
	= - \frac{1}{2\sqrt{2}x\log^{3/2}(1/x)}.
\]
This is negative for $x \in (0,1)$, so $\sum_{\expidx\in\experts} g(\weight{}{\expidx})$ is concave. Again, by Bauer's maximum principle, we have
\*[
	\sum_{\expidx\in\experts} g(\weight{}{\expidx})
	\geq 0,
\]
so
\*[
	\entropyB(\weightvec{})
	\leq \entropyA(\weightvec{}).
\]

For the third inequality, observe that
\*[
	\entropyA(\weightvec{})
	= \sum_{\effexpidx\in\effexperts} \weight{}{\effexpidx} \sqrt{2\log(1/\weight{}{\effexpidx})}
	+ \sum_{\expidx\in\neffexperts} \weight{}{\expidx} \sqrt{2\log(1/\weight{}{\expidx})}.
\]
Consider the convex optimization problem
\*[
  \min_{\substack{\weightdumvec{}\in\PosReals^{\numeffexperts} \\ \inner{\one}{\weightdumvec{}} \leq 1}}
  \bigg\{-\sum_{\effexpidx\in\effexperts} \weightdum{}{\effexpidx} \sqrt{2\log(1/\weightdum{}{\effexpidx})} \bigg\}.
\]
This has the Lagrangian
\*[
	\Ll(\weightdumvec{}; \alpha, \lambda_1,\dots,\lambda_{\numeffexperts})
	= -\sum_{\effexpidx\in\effexperts} \weightdum{}{\effexpidx} \sqrt{2\log(1/\weightdum{}{\effexpidx})} + \alpha(\inner{1}{\weightdumvec{}}-1) + \sum_{\effexpidx\in\effexperts} \lambda_{\effexpidx} \weightdum{}{\effexpidx}.
\]
The partial derivatives for $\effexpidx\in\effexperts$ are
\*[
	\partial_\effexpidx \Ll(\weightdumvec{}; \alpha, \lambda_1,\dots,\lambda_{\numeffexperts})
	= -\sqrt{2\log(1/\weight{}{\effexpidx})} + \frac{1}{\sqrt{2\log(1/\weight{}{\effexpidx})}} + \alpha + \lambda_{\effexpidx}.
\]
Note that this is undefined at $\weight{}{\effexpidx}=0$, so $\lambda_{\effexpidx}=0$ for all $\effexpidx$.
If $\alpha=0$, then for each $\effexpidx\in\effexperts$, $\weight{}{\effexpidx} = e^{-1/2}$, and this is only feasible when $\numeffexperts=1$. In this case,
\*[
	\sum_{\effexpidx\in\effexperts} \weight{}{\effexpidx} \sqrt{2\log(1/\weight{}{\effexpidx})}
	= \frac{1}{\sqrt{e}}.
\]
Otherwise, $\alpha>0$, and since $-\sqrt{2\log(1/x)} + \frac{1}{\sqrt{2\log(1/x)}}$ is monotonic in $x$, all the $\weight{}{\effexpidx}$ are equal. By the K.K.T.\ condition, this implies $\weight{}{\effexpidx} = 1/\numeffexperts$ for each $\effexpidx$, so
\*[
	\sum_{\effexpidx\in\effexperts} \weight{}{\effexpidx} \sqrt{2\log(1/\weight{}{\effexpidx})}
	= \sqrt{2\log\numeffexperts}.
\]
That is, if $\numeffexperts \geq 2$,
\[\label{eqn:ent-conc-bigN0}
	\entropyA(\weightvec{})
	\leq \sqrt{2\log\numeffexperts}
	+ \sum_{\expidx\in\neffexperts} \weight{}{\expidx} \sqrt{2\log(1/\weight{}{\expidx})}.
\]
Next, if $\effexperts=\{\effexpidx\}$, then we have
\[\label{eqn:ent-conc-smallN0}
	\entropyA(\weightvec{})
	&= \weight{}{\effexpidx} \sqrt{2\log(1/\weight{}{\effexpidx})}
	+ \sum_{\expidx\in\neffexperts} \weight{}{\expidx} \sqrt{2\log(1/\weight{}{\expidx})} \\
  &\leq \sqrt{2\weight{}{\effexpidx} \log(1/\weight{}{\effexpidx})}
  + \sum_{\expidx\in\neffexperts} \weight{}{\expidx} \sqrt{2\log(1/\weight{}{\expidx})} \\
	&\leq \sqrt{2(1-\weight{}{\effexpidx})}
	+ \sum_{\expidx\in\neffexperts} \weight{}{\expidx} \sqrt{2\log(1/\weight{}{\expidx})} \\
	&= \sqrt{2\sum_{\expidx\in\neffexperts} \weight{}{\expidx} }
	+ \sum_{\expidx\in\neffexperts} \weight{}{\expidx} \sqrt{2\log(1/\weight{}{\expidx})} \\
	&\leq \sum_{\expidx\in\neffexperts} \left[\sqrt{2\weight{}{\expidx}} + \weight{}{\expidx} \sqrt{2\log(1/\weight{}{\expidx})} \right],
\]
where we have used $x\log(1/x) \leq 1-x$ for all $x>0$ and that $\sqrt{x+y} \leq \sqrt{x} + \sqrt{y}$ for all $x,y>0$.

Finally, for any $p \in (0,1)$, let $f(x) = x^{1-p}\sqrt{2\log(1/x)}$. Observe that $f(0^+)=0$, $f(1) = 0$, and
\*[
	f'(x) = (1-p)x^{-p}\sqrt{2\log(1/x)} - \frac{x^{-p}}{\sqrt{2\log(1/x)}}.
\]
Thus, the only critical point of $f$ occurs at $x_0 = e^{-\frac{1}{2(1-p)}}$. Substituting this gives
\*[
	f(x) \leq \frac{1}{\sqrt{e(1-p)}}.
\]
It follows that for $x\in\cinter{0,1}$
\[
\label{eqn:ent-conc-func}
	x\sqrt{2\log(1/x)} \leq \frac{x^p}{\sqrt{e(1-p)}}
\]
Combining \cref{eqn:ent-conc-bigN0,eqn:ent-conc-smallN0,eqn:ent-conc-func} gives the third inequality.
\end{proof}

\section{Experiments}\label{sec:experiments}

In this section, we provide an empirical evaluation of our proposed algorithms and compare them to existing algorithms in the respective paradigms. In particular, for the quantile regret paradigm we implement the version of \ABNC{} proposed in \cref{example:root_log_regularizer}, while for the semi-adversarial paradigm we implement \ABN{} and compare it to \MetaCARE{} \citep{semiadv} and \OGHedge{} with a decreasing \lrname{} \citep{mourtada2019optimality}. In both cases we generate synthetic data, with specific details in the respective sections.
All algorithms use standard tuning for hyperparameters; the specific details can be found in the code, 
available at \url{https://github.com/blairbilodeau/neurips-2021}.

\subsection{Quantile regret}

For this experiment, we follow the same setup as \citet{chaudhuri2009}. The losses are generated from the \emph{Hadamard} matrix of dimension 64, where the row with constant values is removed, the remaining rows are duplicated with their signs inverted, and all the rows are repeated horizontally in order to get $T=32768$ columns and $n=126$ rows. Each row $i$ represents the losses for expert $i$ from $t=1, \dots, T$. Next, given a parameter $K$, which represents the number of \emph{good} experts (out of $n$), the value 0.025 is subtracted from each entry in the first $K$ rows. Each row is then duplicated $N / n$ times in order to have $N$ total experts. Finally, the values in the matrix are shifted and normalized in order to be bounded in the range $[0, 1]$. For more details, see \citet[Section 3]{chaudhuri2009}. As noted by \citet{chaudhuri2009}, the replication factor affects the behaviour of algorithms that are tuned in terms of the total number of experts $N$, such as \OGHedge{} with a decreasing \lrname{} \citep{mourtada2019optimality} or \ADAHedge{} \citep{derooij14FTL}. On the other hand, parameter-free algorithms such as \SQUINT{} \citep{koolen2015}, \NH{} \citep{chaudhuri2009} or \COINBET{} \citep{orabona2016} are not affected by the replication process. The results illustrated in \cref{fig:quantile} show that \ABNC{} follows the same behaviour as parameter-free algorithms. In particular, we plot the cumulative regret after $T=32768$ time-steps against the best expert. Note that the cumulative loss is equal for the $K \times N / n$ \emph{good} experts, so we are effectively plotting the regret against the top $\varepsilon=K/n$-quantile, which is kept constant despite the total number of experts $N$ increasing. We adopt an initial uniform distribution over experts for all the algorithms.

\begin{figure}[t]
\centering
\includegraphics[width=.45\linewidth]{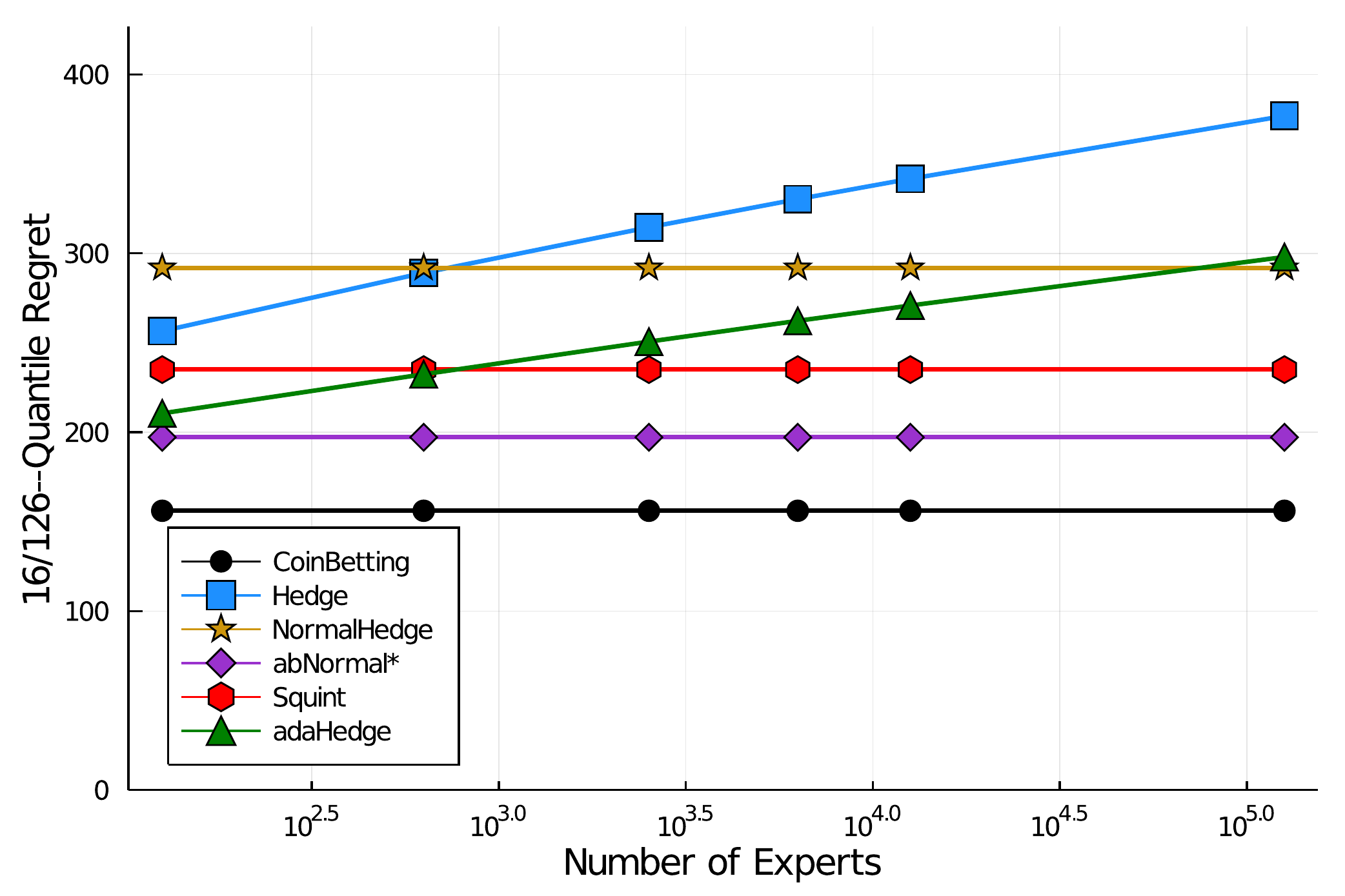}
\includegraphics[width=.45\linewidth]{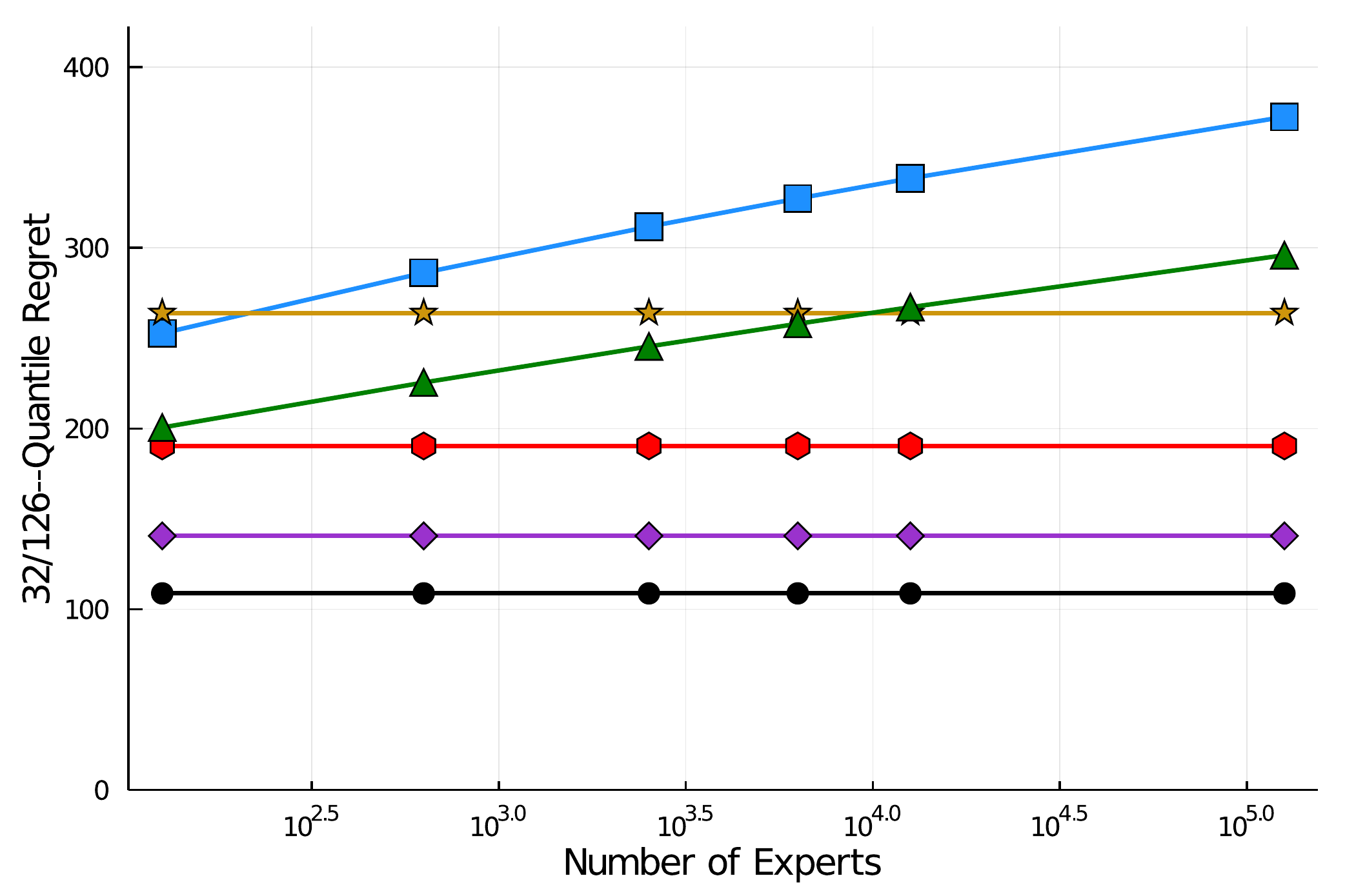}
\caption{Quantile regret to the top $K/n$ proportion of experts versus number of experts $\numexperts$.
 }
\label{fig:quantile}
\end{figure}

\subsection{Semi-adversarial losses}

For this experiment, we compare \OGHedge{} with decreasing \lrname{} $\eta_t \propto 1 / \sqrt{t}$, \MetaCARE{} \citep{semiadv} and \ABN{} from \cref{sec:semiadv-main-results}. In this case, we keep the number of total experts fixed to $\numexperts = 1000$, and use deterministic losses. When $\numeffexperts=1$, the best expert has loss 0.4 and the rest have loss 0.5 on every round. When $\numeffexperts=2$, expert 1 alternates between loss of 0 and loss of 1, expert 2 alternates (on opposite rounds from expert 1) between loss of 1 and loss of 0, and the rest have loss 0.6 on every round.
Thus, both $\numeffexperts=1$ and $\numeffexperts=2$ correspond to $\Deltaeff=0.1$. 
When $\numeffexperts=\numexperts$, the first half and second half of the experts alternate between loss 1 and loss 0 on opposite rounds.

As the theory of \citet{semiadv} prescribes, \OGHedge{} does well for $\numeffexperts=1$ but does equally poorly for $\numeffexperts=2$ as $\numeffexperts=\numexperts$, while \MetaCARE{} improves on \OGHedge{} for $1 < \numeffexperts < \numexperts$. As \cref{thm:discrete-semiadv} prescribes, \ABN{} does well in all settings, improving on \MetaCARE{} by adapting faster in the $\numeffexperts=2$ case.

\begin{figure}[t]
\centering
\includegraphics[scale=0.35]{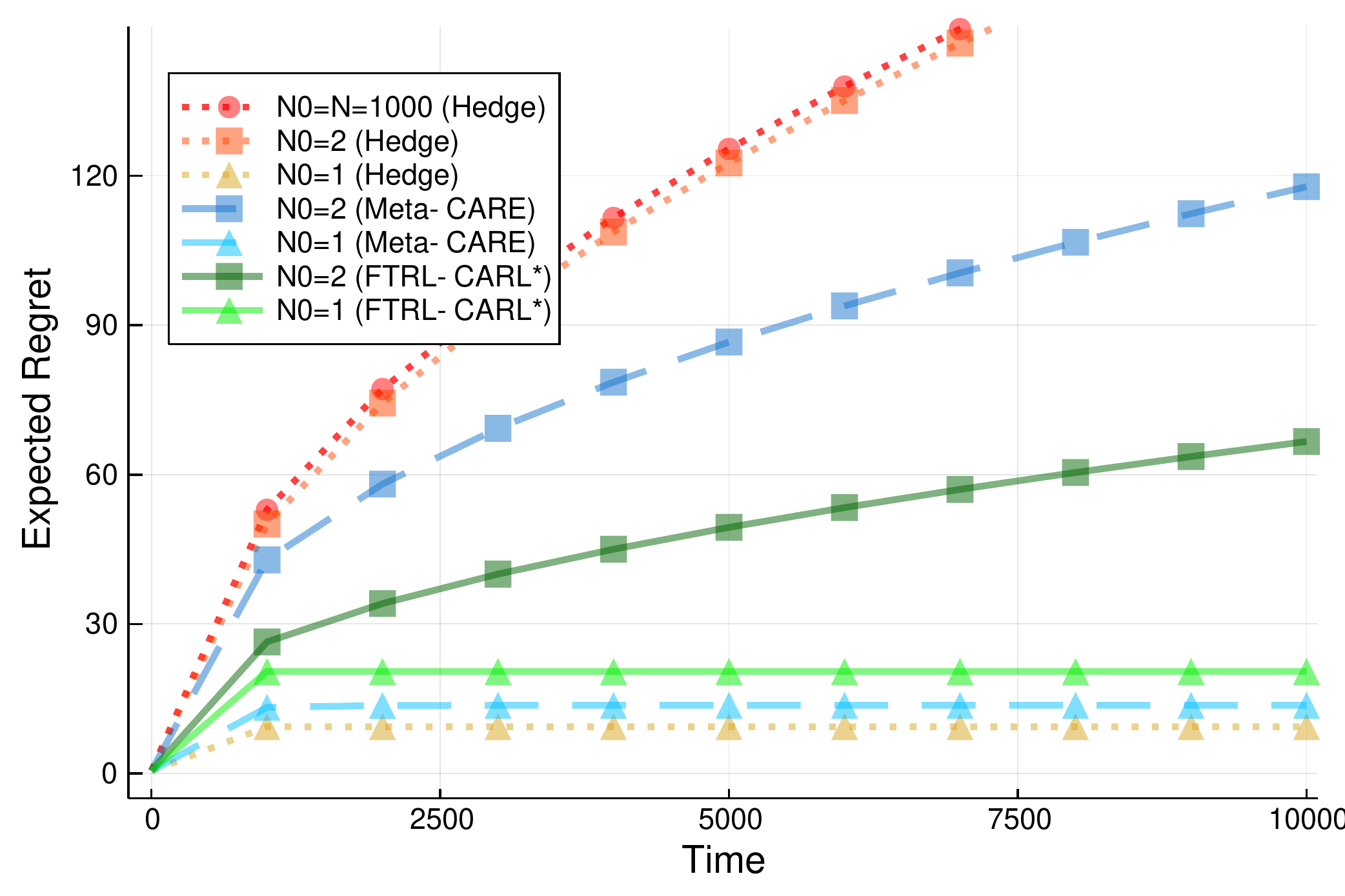}
\caption{Semi-adversarial paradigm: Expected regret versus time for various values of $\numeffexperts$.}
\label{fig:semi-adv}
\end{figure}

\end{document}